\pdfoutput=1

\documentclass[11pt]{article}

\PassOptionsToPackage{dvipsnames}{xcolor}
\usepackage{assets/acl}

\usepackage{adjustbox}
\usepackage{multirow}
\usepackage{wrapfig}
\usepackage{scalerel}
\usepackage{float}
\usepackage[caption = false]{subfig}

\usepackage[T1]{fontenc}    

\usepackage[utf8]{inputenc}

\usepackage{microtype}
\usepackage{svg}

\usepackage{inconsolata}

\usepackage{times}
\usepackage{latexsym}
\usepackage{datetime}       
\usepackage{booktabs}       
\usepackage{nicefrac}       
\usepackage{xcolor}         
\usepackage[linguistics]{forest}  
\usepackage{url}            
\usepackage{fnpct}           
\usepackage{enumitem}        
\usepackage{tcolorbox}       
\usepackage{csquotes}        
\usepackage{xspace}          
\makeatletter
\xspaceaddexceptions{\csq@qclose@i}
\makeatletter
\usepackage{tikz}
\usepackage[normalem]{ulem} 

\usepackage[scaled=0.85]{helvet}

\usepackage{bbm}
\usepackage{amsfonts}
\usepackage{amssymb}
\usepackage{mathtools}
\usepackage{amsmath}
\usepackage{bm}
\usepackage{amsthm}
\usepackage{thmtools}
\usepackage{thm-restate}
\usepackage{cancel}

\newtheorem{definition}{Definition}
\newtheorem{assumption}{Assumption}
\newtheorem{estimator}{Estimator}
\newtheorem{lemma}{Lemma}

\DeclareMathOperator*{\argmax}{argmax}

\usepackage{nicefrac}

\usepackage[createShortEnv,commandRef=Cref]{proof-at-the-end}

\usepackage{algorithm}
\usepackage{algpseudocode}  
\makeatletter
\algnewcommand{\LineComment}[1]{\Statex \hskip\ALG@thistlm \textcolor{blue}{\(\triangleright\) #1}}
\algnewcommand{\FirstLineComment}[1]{\Statex \hskip\ALG@tlm \textcolor{blue}{\(\triangleright\) #1}}
\algnewcommand{\InlineComment}[1]{\hfill\textcolor{blue}{\(\triangleright\) #1}}
\makeatother

\usepackage{cleveref}
\crefname{section}{\S}{\S\S}
\Crefname{section}{\S}{\S\S}
\crefformat{section}{\S#2#1#3}
\crefname{figure}{Fig.}{Fig.}
\crefname{alg}{Alg.}{Alg.}
\crefname{line}{line}{lines}
\crefname{appendix}{App.}{App.}
\crefname{equation}{eq.}{eqs.}
\crefname{table}{Table}{Tables}
\crefname{proposition}{Proposition}{Propositions}
\crefname{assumption}{Assump.}{Assumps.}
\crefname{lemma}{Lemma}{Lemmas}
\crefname{definition}{Defn.}{Defns.}
\crefname{estimator}{Estimator}{Estimators}
\crefname{theorem}{Theorem}{Theorems}
\crefname{thm}{Theorem}{Theorems}

\usepackage{tabularray}
\usepackage{cellspace}
\newcommand\cincludegraphics[2][]{\raisebox{-0.3\height}{\includegraphics[#1]{#2}}}
\usepackage{setspace}

\usepackage{siunitx}
\newcommand{\q}[2]{\qty[mode=math]{#1}{#2}\xspace}

\DeclareSIUnit[quantity-product = {}, reset-math-version = false]\thousand{k}
\DeclareSIUnit[quantity-product = {}, reset-math-version = false]\million{M}
\DeclareSIUnit[quantity-product = {}, reset-math-version = false]\billion{B}
\DeclareSIUnit[quantity-product = {}, reset-math-version = false]\trillion{T}

\DeclareSIUnit[quantity-product = {}, reset-math-version = false]\x{x}
\DeclareSIUnit[quantity-product = {}, reset-math-version = false]\percent{\%}

\DeclareSIUnit[quantity-product = {}, reset-math-version = false]\hour{h}
\DeclareSIUnit[quantity-product = {}, reset-math-version = false]\min{m}
\DeclareSIUnit[quantity-product = {}, reset-math-version = false]\sec{s}

\newcommand{\integer}[1]{%
    \num[
        mode = math,
        round-mode=places,
        round-precision=0,
        group-separator={,},
        group-minimum-digits=4,
    ]{#1}%
    \xspace
}

\newcommand{\float}[2][1]{%
    \num[
        group-digits=false,
        round-precision=#1,
        round-mode=places,
    ]{#2}%
    \xspace}

\xspaceaddexceptions{\textsubscript}

\makeatletter
\DeclareRobustCommand*{\escapeus}[1]{%
    \begingroup\@activeus\scantokens{#1 }\endgroup}
\begingroup\lccode`\~=`\_\relax
\lowercase{\endgroup\def\@activeus{\catcode`\_=\active \let~\_}}
\makeatother

\newcommand{\makesf}[1]{\textsf{{\escapeus{#1}}}}

\usepackage[textsize=tiny]{todonotes}

\newcommand{\defn}[1]{\textbf{#1}}

\newcommand*{\circled}[1]{\tikz[baseline=(char.base)]{
        \node[shape=circle,draw,inner sep=1pt] (char) {\normalfont{\small #1}};}}
\newcommand{\instancecolor}{ForestGreen}
\newcommand{\cohortcolor}{orange}
\newcommand{\checkpointcolor}{purple}

\newcommand{\size}[1]{{\vert #1 \vert}}

\newcommand{\perffn}{\gamma}
\newcommand{\permutationfn}{\sigma}

\newcommand\mem{\tau}
\newcommand{\memorisationinstance}{\mem_{\xdatainstance, \checkpointtime}}
\newcommand{\memorisationinstancelastcheckpoint}{\mem_{\xdatainstance, \scaleto{\lastcheckpointtime}{4pt}}}
\newcommand{\memorisationcohort}{\mem_{\cohorttime, \checkpointtime}}
\newcommand{\memorisationarchitecture}{\mem_{\xdatainstance,\scaleto{p(\trainingvar)}{6pt}}^{\mathtt{arch}}}
\newcommand{\memorisationarchitectureestimator}{\widehat{\mem}_{\xdatainstance,\scaleto{p(\trainingvar)}{6pt}}^{\mathtt{arch}}}

\newcommand{\vocab}{\mathcal{V}}
\DeclareMathOperator*{\expect}{\mathbb{E}}
\newcommand{\vx}{{\color{\instancecolor} \boldsymbol{x}}}
\newcommand{\vtheta}{\boldsymbol{\theta}}
\newcommand{\ptheta}{p_{\scaleto{\vtheta}{4pt}}}
\newcommand{\dataset}{\mathcal{D}}
\newcommand{\R}{\mathbb{R}}
\newcommand{\batch}{\mathcal{B}}
\newcommand{\loss}{\mathcal{L}}

\newcommand{\defeq}[0]{\mathrel{\stackrel{\textnormal{\tiny def}}{=}}}
\newcommand{\datainstance}{n}
\newcommand{\xdatainstance}{\vx}
\newcommand{\smalldots}{...}
\newcommand{\cohorttime}{{\color{\cohortcolor}g}}
\newcommand{\Cohorttime}{G}
\newcommand{\checkpointtime}{{\color{\checkpointcolor}c}}
\newcommand{\checkpointtimesec}{{\color{\checkpointcolor}c'}}
\newcommand{\lastcheckpointtime}{{\color{\checkpointcolor}T}}

\newcommand{\batchsize}{B}
\newcommand{\cohortcoloredt}{{\color{\cohortcolor}t}}
\newcommand{\checkpointcoloredt}{{\color{\checkpointcolor}t}}
\newcommand{\trainingvar}{\boldsymbol{\psi}}

\newcommand{\yestimator}{\overline{Y}}

\newcommand{\didestimand}{\memorisationcohort^{\mathtt{did}}}
\newcommand{\diffestimand}{\memorisationcohort^{\mathtt{diff}}}
\newcommand{\didestimator}{\widehat{\mem}_{\cohorttime, \checkpointtime}^{\mathtt{did}}}
\newcommand{\diffestimator}{\widehat{\mem}_{\cohorttime, \checkpointtime}^{\mathtt{diff}}}

\newcommand{\macrobatch}{\mathcal{G}}
\newcommand{\notseen}{{\color{\cohortcolor} \infty}}
\newcommand{\pthetacheckpoint}{p_{\scaleto{\vtheta_{\checkpointtime}}{4pt}}}
\newcommand{\mathcomment}[1]{\text{{\color{gray} #1}}}
\newcommand{\lagcohorttime}{\mathop{{\color{\checkpointcolor}\cohorttime-1}}}

\newcommand{\cohort}{macro-batch\xspace}
\newcommand{\cohorts}{macro-batches\xspace}
\newcommand{\did}{DiD\xspace}
\newcommand{\kl}{\ensuremath{(k,\!\ell)}\xspace}
\newcommand{\klestimand}{\mem_{\xdatainstance, \checkpointtime}^{\mathtt{extr}}}
\newcommand{\klestimator}{\widehat{\mem}_{\xdatainstance, \checkpointtime}^{\mathtt{extr}}}
\newcommand{\memarchitecture}{\mem_{\xdatainstance,\scaleto{p(\trainingvar)}{6pt}}}
\newcommand{\paramswithx}{\Theta_{\cohorttime}}
\newcommand{\paramswithoutx}{\Theta_{\notseen}}
\newcommand{\lastcheckpointtimesmall}{\scaleto{\lastcheckpointtime}{4pt}}
\newcommand{\optvthetawithx}{\vtheta_{\scaleto{\lastcheckpointtime}{4pt}}}
\newcommand{\optvthetawithoutx}{\vtheta_{-\xdatainstance, \scaleto{\lastcheckpointtime}{4pt}}}
\newcommand{\memorisationinfluenceestimator}{\hat{\mem}_{\xdatainstance,\scaleto{\lastcheckpointtime}{4pt}}^{\mathtt{infl}}}
\newcommand{\hessian}{\mathrm{H}}
\newcommand{\Ywithoutx}{Y_{-\xdatainstance, \scaleto{\lastcheckpointtime}{4pt}}}

\newcommand{\papertitle}{Causal Estimation of Memorisation Profiles}
\title{\papertitle}

\newcommand{\camid}{{\includegraphics[scale=0.018]{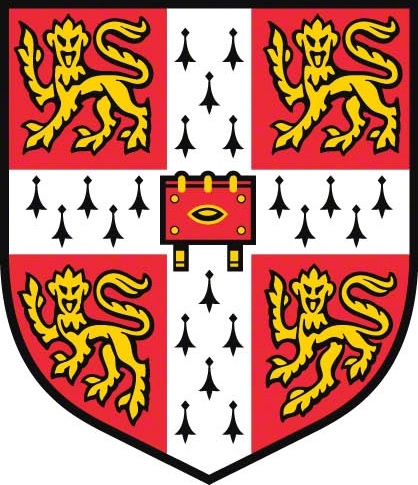}}}
\newcommand{\ethid}{{\includegraphics[scale=0.028]{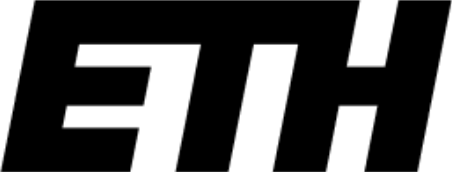}}}
\newcommand{\camemailadress}[1]{\href{mailto:#1@cam.ac.uk}{\makesf{#1}}}
\newcommand{\ethemailadress}[1]{\href{mailto:#1@inf.ethz.ch}{\makesf{#1}}}

\author{Pietro Lesci,\textsuperscript{\camid} Clara Meister,\textsuperscript{\ethid} Thomas Hofmann,\textsuperscript{\ethid} Andreas Vlachos,\textsuperscript{\camid} Tiago Pimentel\textsuperscript{\ethid} \\
  $^\camid$University of Cambridge,~\,~ \textsuperscript{\ethid}ETH Z\"urich \\
  \makesf{\{}\camemailadress{pl487}, \camemailadress{av308}\makesf{\}@cam.ac.uk} \\
  \makesf{\{}\ethemailadress{clara.meister}, \ethemailadress{thomas.hofmann}, \ethemailadress{tiago.pimentel}\makesf{\}@inf.ethz.ch}
}

\begin{document}

\maketitle

\begin{abstract}
    Understanding memorisation in language models has practical and societal implications, e.g., studying models' training dynamics or preventing copyright infringements.
    Prior work defines memorisation as the causal effect of training with an instance on the model's ability to predict that instance.
    This definition relies on a counterfactual: the ability to observe what would have happened had the model not seen that instance.
    Existing methods struggle to provide computationally efficient and accurate estimates of this counterfactual.
    Further, they often estimate memorisation for a model architecture rather than for a specific model instance.
    This paper fills an important gap in the literature, proposing a new, principled, and efficient method to estimate memorisation based on the difference-in-differences design from econometrics.
    Using this method, we characterise a model's memorisation profile---its memorisation trends across training---by only observing its behaviour on a small set of instances throughout training.
    In experiments with the Pythia model suite, we find that memorisation (i) is stronger and more persistent in larger models, (ii) is determined by data order and learning rate, and (iii) has stable trends across model sizes, thus making memorisation in larger models predictable from smaller ones.
\end{abstract}

\begin{tblr}{colspec = {Q[c,m]|X[l,m]}, stretch = 0}
    \cincludegraphics[width=1.2em, keepaspectratio]{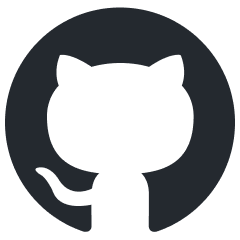}
     & \setstretch{.5}\href{https://github.com/pietrolesci/memorisation-profiles}{{\makesf{pietrolesci/memorisation-profiles}}} \\
\end{tblr}

\section{Introduction}\label{sec:intro}

Large language models (LMs) are often pretrained with a single pass on web-scale datasets \citep[\emph{inter alia}]{raffel-etal-2020-exploring, gao-etal-2020-pile, penedo-etal-2023-refinedweb}.
Given the colossal size of these training sets, one may expect each individual instance to have little impact on the final model.
Yet, LMs can still reproduce entire sequences from their training set verbatim \citep{carlini-etal-2021-extracting}, suggesting that models can store, or \emph{memorise}, precise knowledge about individual training instances.
In the era of large LMs, measuring memorisation is crucial for NLP practitioners; it has implications for copyright and data protection \citep{hu-etal-2022-membership,vyas2023,lee2023}, for how models encode factual information \citep{cao-etal-2022-benign, tirumala-etal-2022-memorization}, and for understanding their training dynamics \cite{arpit2017closer,chang_survey_2024}.

\begin{figure}[!t]
    \centering
    \includegraphics[width=\columnwidth]{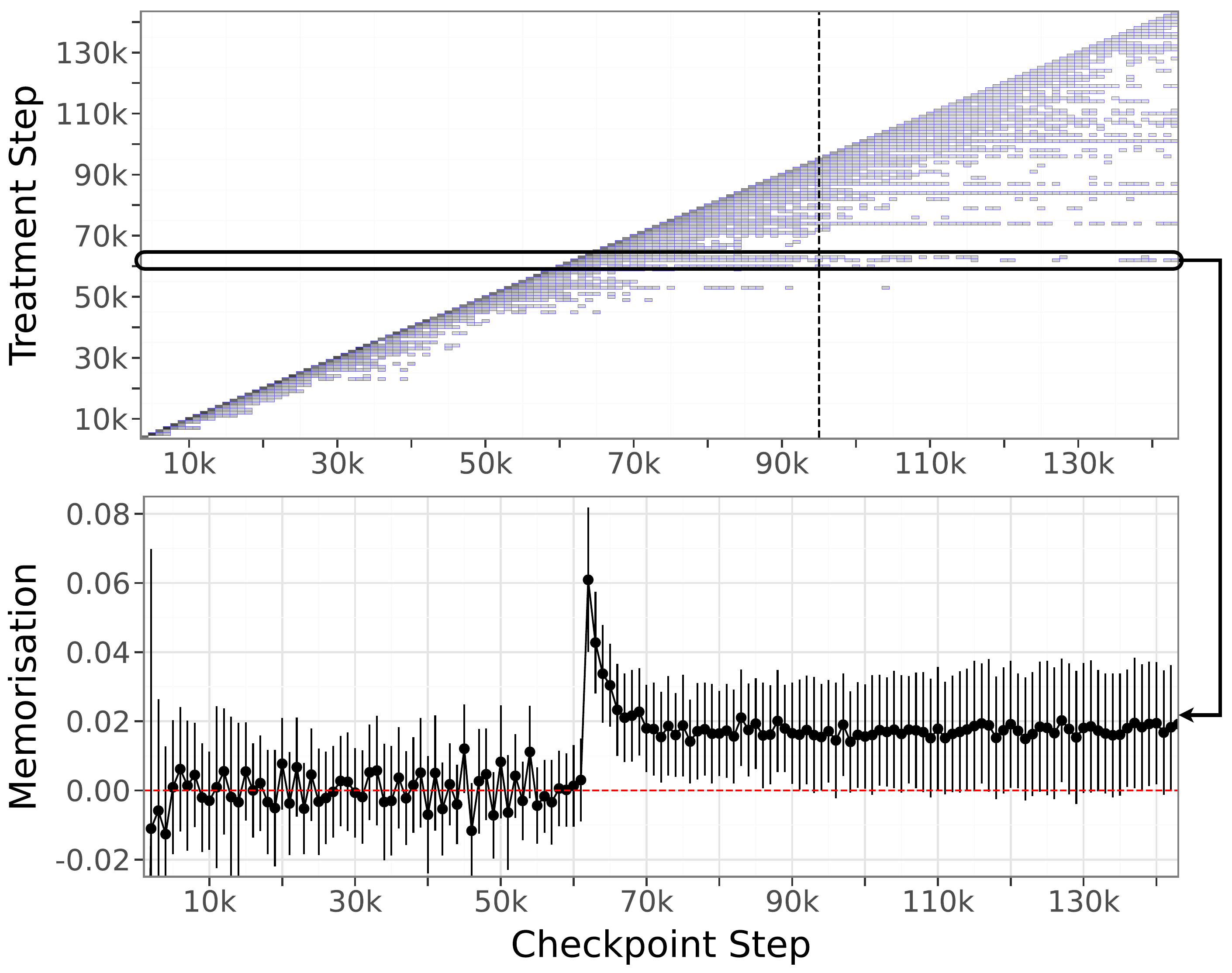}
    \vspace{-15pt}
    \caption{Memorisation profile (top) and path (bottom) of Pythia \q{6.9}{\billion}.
        Each entry represents the expected counterfactual memorisation of instances trained on at a specific timestep (\enquote{Treatment Step}) across model checkpoints (\enquote{Checkpoint Step}). The dashed vertical line indicates the end of the first epoch.}
    \label{fig:memorisation_profile}
\end{figure}

One line of prior work has adopted a causal definition of \defn{memorisation}: it is the causal effect of observing an instance during training on a model's ability to correctly predict that instance \citep{feldman-2020-does}.
Despite being an intuitive concept, quantification of this definition is not straightforward as it requires knowledge of a \emph{counterfactual}:\footnote{Counterfactuals are thought experiments about what would have happened if a present condition were changed while keeping everything else unchanged \citep{mccloskey-2016-counterfactuals}.}
we must know how our model would have performed on a training instance had the model not been trained on it.
To overcome this challenge, prior work has proposed a variety of methods to estimate memorisation.
Some estimate it by training a model multiple times on different subsets of the data \citep[e.g.,][]{feldman-zhang-2020-what,zheng-jiang-2022-empirical}, while others implicitly assume this counterfactual's value to be negligible \citep{carlini-etal-2021-extracting}.
Both these approaches, however, have drawbacks:
the first computes memorisation for an architecture rather than a specific model, while the second relies on a strong assumption (we discuss this in detail in \cref{sec:related_work}).

In this paper, we first formalise \defn{counterfactual memorisation} as the difference between two potential outcomes; notably, our formalisation generalises prior definitions of memorisation, allowing us to compare them within a unified framework.
We then draw from the econometrics literature \citep{callaway-santanna-2021-difference} and propose a new method which estimates memorisation using only observational data;
our method simply needs a model's performance measurements (e.g., log-likelihood) on a subset of the training data throughout training.
The output of our method is what we term a \defn{memorisation profile}: a model's memorisation of training batches over the course of training.

Empirically, we use this method to analyse memorisation for models in the Pythia suite  \citep{biderman-etal-2023-pythia} and characterise their memorisation profiles; e.g.,
\cref{fig:memorisation_profile} reports the memorisation profile of Pythia \q{6.9}{\billion}.
By studying these memorisation profiles we find that memorisation is stronger and more persistent in larger models.
Furthermore, both the learning rate and the position of an instance in the training set considerably impact how strongly that instance is memorised.
Finally, memorisation profiles are stable across model sizes; thus, we can predict memorisation in larger models from the memorisation observed in smaller ones.

\section{Background}

In this section, we introduce some background on language modelling and on causal analysis which will be required throughout our paper.

\subsection{Language Modelling}\label{sec:langmodel}

We start by providing some background on language modelling.\footnote{We frame our exposition in terms of language models, but our framework can be trivially applied to any neural model and input modality (e.g., images).}
Let $\ptheta(\vx)$ be a language model with parameters $\vtheta \in \R^d$.
This model defines a probability distribution over $\vx \in \vocab^*$, the set of all finite sequences that can be constructed from elements in the alphabet $\vocab$.
To train this model, we start with a set of randomly selected initial parameters, $\vtheta_0$.
We then learn parameters $\vtheta$ using a dataset $\dataset$ and an optimisation procedure defined with respect to a loss function $\loss$.

Specifically, let $\dataset = \{\xdatainstance_{\datainstance}\}_{n=1}^{N}$ be a dataset whose \defn{\color{\instancecolor} instances} $\xdatainstance$ are sequences drawn from a target (unknown) distribution $p(\vx)$.
These instances are typically assumed to be sampled i.i.d., and are shuffled using a permutation function $\permutationfn\colon \{1,\smalldots, N\} \!\to\! \{1,\smalldots, N\}$.
For a given batch size $\batchsize$, we then split this dataset into $T \!\leq\! \lfloor \nicefrac{N}{B} \rfloor$ batches $\batch_{\cohortcoloredt}$.
We iterate through these batches performing gradient updates on the model parameters:
\begin{align}\label{eq:optimisation_process}
    \vtheta_{\checkpointcoloredt} = \vtheta_{{\color{\checkpointcolor} t-1}} - \eta\,
    \nabla_{\vtheta}
    \loss(\vtheta_{{\color{\checkpointcolor} t-1}}, \batch_{\cohortcoloredt})
\end{align}
where $\eta \in \R$ is a learning rate.\footnote{The learning rate can be a function of other quantities \citep[e.g.,][\emph{inter alia}]{adagrad,adam}.}
Notably, this procedure consists of a single pass on the training set and is standard for recent LMs \citep[e.g.,][]{touvron-etal-2023-llama2,jiang-etal-2023-mistral,dey-etal-2023-cerebras}.

At each iteration, we use a batch $\batch_{\cohortcoloredt}$ to obtain a new model checkpoint, $\vtheta_{\checkpointcoloredt}$.
We introduce a notation which distinguishes the indexing of checkpoints and batches.
We use $\checkpointtime \in \{0, 1, \smalldots, T\}$ to denote a \defn{\color{\checkpointcolor} checkpoint step} (e.g., $\vtheta_\checkpointtime$). Further, we use $\cohorttime \in \{1, \smalldots, T\} \cup \{\infty\}$ to denote the timestep at which a batch is used for training (e.g., $\batch_\cohorttime$); we term this a \defn{\color{\cohortcolor} treatment step}, borrowing this terminology from the econometrics literature.
We denote as $\cohorttime\!=\!\infty$ a batch composed of instances that are not used for training and which form a validation set.

\subsection{Causal Analysis}\label{sec:causal_analysis}

Causal estimation is typically split into three steps.
First, we define a \defn{causal estimand}, the target quantity we want to estimate.
Second, we state the assumptions needed to rewrite this causal estimand in terms of observable data, thus defining a \defn{statistical estimand}; this process is called identification.
Finally, we define an \defn{estimator}, a statistical procedure to approximate the statistical estimand.

To formally define memorisation as a causal estimand, we will use the \defn{potential outcomes} framework of \citet{rubin-1974-estimating, rubin-2005-causal}.\footnote{For a comparison of causal frameworks, see \citet{ibeling-and-icard-2023-comparing}. For an introduction to causal inference, see \citet{pearl-2009-causality} and \citet{imbens-rubin-2015-causal}.}
This framework allows us to formally describe the causal effect of an intervention, or \defn{treatment}, on some target quantity, or \defn{outcome}.
In \cref{sec:intro}, we defined memorisation as the causal effect of \emph{training on} an instance on a model’s ability to \emph{predict} it.
Thus, the act of training on $\xdatainstance$ defines the treatment, while the model's ability to predict an instance defines the outcome.

Since training is performed iteratively over batches, instances are treated at different timesteps.
Thus, we use a \defn{treatment assignment variable} $\Cohorttime(\xdatainstance)$ to denote the step $\cohorttime$ an instance is trained on.
Further, to quantify the ability of a model with parameters $\vtheta_\checkpointtime$ to predict $\xdatainstance$, we use a performance function $\perffn$.
We then define the \defn{outcome variable} as $Y_\checkpointtime(\xdatainstance) \defeq \perffn(\vtheta_{\checkpointtime}, \xdatainstance)$, and, unless noted otherwise, we set this performance function to be the log-likelihood of $\vx$ under $\pthetacheckpoint$: $\perffn(\vtheta_{\checkpointtime}, \xdatainstance) = \log \pthetacheckpoint(\xdatainstance)$.\footnote{We experiment with other functions in the appendix.}

To define memorisation we need to represent both observed---i.e., $Y_\checkpointtime(\xdatainstance)$---and counterfactual outcomes---i.e., the performance of the model on $\xdatainstance$ had we not trained on it.
The potential outcomes notation \citep{splawa1990application} enables us to represent both types of outcomes consistently.

\begin{definition}
    The \defn{potential outcome} of an instance $\xdatainstance$ at timestep $\checkpointtime$ under treatment assignment $\cohorttime$, denoted as $Y_\checkpointtime(\xdatainstance;\cohorttime)$, is the value that the outcome would have taken if $\Cohorttime(\xdatainstance)$ was equal to $\cohorttime$.
\end{definition}

Since we only observe a single permutation of the data, we only see one specific treatment step for each instance, i.e., $\Cohorttime(\xdatainstance)$.
Thus, the potential outcome of an instance is observed only for the actual treatment assignment $\cohorttime\!=\!\Cohorttime(\xdatainstance)$.
In this case, we can equate potential and observed outcomes, that is $Y_{\checkpointtime}(\xdatainstance; \cohorttime) = Y_{\checkpointtime}(\xdatainstance)$, a property called consistency \citep{cole-and-frangakis-2009-consistency}.
For any other treatment step $\cohorttime \mathop{\neq} \Cohorttime(\xdatainstance)$, the potential outcome is counterfactual and, thus, unobservable from the data.\looseness=-1

\section{Counterfactual Memorisation}\label{sec:mem_definitions}

Intuitively, counterfactual memorisation can be understood as the answer to the question: how would the model's performance on instance $\xdatainstance$ at timestep $\checkpointtime$ be different if we had not trained on it at timestep $\cohorttime$?
Using the potential outcomes notation, we formalise this definition as follows.

\begin{definition}\label{defn:instance_memorisation}
    \defn{Counterfactual memorisation} is the causal effect of using instance $\xdatainstance$ for training at the observed timestep $\Cohorttime(\xdatainstance)\mathop{=}\cohorttime$ on the model's performance on this same instance at timestep $\checkpointtime$:
    \begin{align}\label{eq:instance_memorisation}
        \memorisationinstance\,\,\, \defeq \underbrace{Y_{\checkpointtime}(\xdatainstance; \cohorttime)}_{\substack{\text{performance on $\xdatainstance$} \\ \text{when trained with $\xdatainstance$}}} - \underbrace{Y_{\checkpointtime}(\xdatainstance; \notseen)}_{\substack{\text{performance on $\xdatainstance$}\\ \text{when not trained with $\xdatainstance$}}} \!\!\!\!
    \end{align}
\end{definition}

In econometrics, \cref{eq:instance_memorisation} is called an individual treatment effect (ITE).
Notably, the first potential outcome in this equation, $Y_{\checkpointtime}(\xdatainstance; \cohorttime)$, can be observed from the data since, by definition, we trained on $\xdatainstance$ at timestep $\Cohorttime(\xdatainstance)\mathop{=}\cohorttime$.
However, the second term, $Y_{\checkpointtime}(\xdatainstance; \notseen)$, is counterfactual.
To compute the ITE, we would need to estimate this counterfactual outcome for a specific instance, which is challenging due to unobserved factors and heterogeneity\footnote{I.e.,\ non-random variability across instances.} \citep{lu-etal-2018-estimating}.
While we would ideally estimate memorisation at the instance level, we focus on average effects instead, as is common in the econometrics literature \citep{angrist-pischke-2015-mastering}.

\begin{definition}\label{defn:expect_memorisation}
    \defn{Expected counterfactual memorisation} is the average causal effect of using instances for training at timestep $\cohorttime$ on the model's performance on these same instances at timestep $\checkpointtime$:\footnote{In econometrics, \cref{eq:expect_memorisation} is called an average treatment effect on the treated (ATT), as it is defined in terms of an expectation over $\xdatainstance \sim p(\vx \mathop{\mid} \Cohorttime(\vx) \mathop{=}\cohorttime)$.
        In other words, this expectation is taken with respect to the instance distribution conditioned on it being selected for training at step $\cohorttime$.
        Assuming the training set is sampled i.i.d.\ and that its permutation is random (as discussed in \cref{sec:langmodel}), then $p(\vx \mathop{\mid} \Cohorttime(\vx) \mathop{=} \cohorttime) = p(\vx)$.
        Given these assumptions, \cref{eq:expect_memorisation} would also be an average treatment effect (ATE), which would allow us to make causal claims about the entire population.}
    \begin{align}\label{eq:expect_memorisation}
        \memorisationcohort \defeq \expect_{\xdatainstance}
        \Big[Y_{\checkpointtime}(\xdatainstance; \cohorttime)
            - Y_{\checkpointtime}(\xdatainstance; \notseen) \mid \Cohorttime(\xdatainstance) \mathop{=} \cohorttime\Big]
    \end{align}
\end{definition}

Together, the $\memorisationcohort$ form a matrix which we term the model's \defn{memorisation profile};
each row therein is the \defn{memorisation path} of a batch.
Memorisation profiles and paths allow us to analyse a model's memorisation patterns across different treatment and checkpoint steps.
Notably, there cannot be memorisation whenever $\checkpointtime < \cohorttime$, as the instances have not been seen by the model yet, so  $\memorisationcohort = 0$ in those cases.
We term $\memorisationcohort$ as \defn{instantaneous memorisation} when $\checkpointtime=\cohorttime$, as \defn{persistent memorisation} when $\checkpointtime>\cohorttime$, and as \defn{residual memorisation} when $\checkpointtime = \lastcheckpointtime$.

\section{Estimating Memorisation}\label{sec:estimation}

The practical implication of defining memorisation at a treatment level $\cohorttime$ is that we can only make causal claims for groups of instances treated at the same timestep (i.e., a batch $\batch_\cohorttime$), rather than for individual instances.
However, even though this approach simplifies the problem, estimating \cref{eq:expect_memorisation} still poses major challenges as it contains a counterfactual.
A simple decomposition makes this counterfactual explicit:%
\begin{align}\label{eq:expect_memorisation_breakdown}
     & \memorisationcohort =
    \\
     & \underbrace{\expect_{\xdatainstance}\! \big[Y_{\checkpointtime}(\xdatainstance; \cohorttime)\mathop{\mid} \Cohorttime(\xdatainstance) \mathop{=} \cohorttime\big]\!}_{\circled{1}} \mathop{-} \underbrace{\expect_{\xdatainstance} \!\big[Y_{\checkpointtime}(\xdatainstance; \notseen)\mathop{\mid}\Cohorttime(\xdatainstance) \mathop{=} \cohorttime\big]\!}_{\circled{2}} \nonumber
\end{align}
Expectation $\circled{1}$ can be directly estimated from the data because batch $\batch_\cohorttime$ contains a set of examples $\xdatainstance$ for which $\Cohorttime(\xdatainstance)\mathop{=}\cohorttime$ and, thus, we can invoke the consistency assumption to equate $Y_{\checkpointtime}(\xdatainstance; \cohorttime)$ with the observed outcome $Y_{\checkpointtime}(\xdatainstance)$.
Let us define the mean across instances in a batch as:
\begin{align}\label{eq:mean_Y}
    \yestimator_{\checkpointtime}(\cohorttime) \defeq \frac{1}{|\batch_\cohorttime|} \sum_{\xdatainstance \in \batch_\cohorttime} Y_{\checkpointtime}(\xdatainstance)
\end{align}
We can thus use \cref{eq:mean_Y} as an estimator for expectation $\circled{1}$. Expectation $\circled{2}$, however, is counterfactual: we cannot observe the potential outcome $Y_{\checkpointtime}(\xdatainstance; \notseen)$ for instances treated at timestep $\cohorttime\mathop{\neq}\notseen$.
The presence of counterfactual potential outcomes in causal estimands creates challenges for their estimation, being known as the fundamental problem of causal inference \citep{holland-1986-statistics}.
The goal of causal methods is then to estimate these counterfactual outcomes from observed ones, using comparable groups of instances.
Thus far, we have defined our causal estimand.
In this section, we perform steps two and three of causal estimation (\cref{sec:causal_analysis}): we derive two statistical estimands for our causal estimand (identifying it under specific assumptions), and provide concrete estimators for them.

\subsection{The Difference Estimator}\label{sec:estimation:diff}

Our first approach to estimate memorisation is straightforward and only requires the observed outcomes of a held-out validation set.
However, it relies on a strong identification assumption.

\begin{assumption}[I.i.d.\ Dataset Sampling]\label{ass:iid_sampling}
    Instances $\xdatainstance$ are independently and identically distributed, following $p(\vx)$, and are randomly assigned to treatment groups $\cohorttime$.\looseness=-1
\end{assumption}

Under this assumption, we have that $p(\vx \mathop{\mid} \Cohorttime(\vx) \mathop{=}\cohorttime) \!=\! p(\vx \mathop{\mid} \Cohorttime(\vx) \mathop{=}\notseen) \!=\! p(\vx)$.
Thus, the following statistical estimand identifies $\memorisationcohort$:\looseness=-1%
\begin{align}\label{eq:diff_identification}
    \diffestimand = \expect_{\xdatainstance} \big[ & Y_{\checkpointtime}(\xdatainstance; \cohorttime) \mid \Cohorttime(\xdatainstance) \mathop{=} \cohorttime \big]                                            \\
                                                   & - \expect_{\xdatainstance} \big[Y_{\checkpointtime}(\xdatainstance; \notseen) \mid \mathop{\Cohorttime(\xdatainstance) \mathop{=} \notseen}\big]\nonumber
\end{align}
(See \cref{lemma:identification_difference_estimand} in \cref{app:proof_diff_estimand} for a proof.)
Note that, unlike \cref{eq:expect_memorisation}, the second term in this estimand is not counterfactual: it is the expected observed outcome of validation instances, $\batch_{\notseen}$.
This statistical estimand can then be estimated as follows.

\begin{estimator}\label{def:differenceestimator}
    The \defn{difference estimator}, defined as:
    \begin{equation}\label{eq:difference estimator}
        \diffestimator = \yestimator_{\checkpointtime}(\cohorttime) - \yestimator_{\checkpointtime}(\notseen)
    \end{equation}
    is an unbiased estimator of $\memorisationcohort$ under \cref{ass:iid_sampling}.
\end{estimator}
\begin{proof}
    \vspace{-2pt}
    See \cref{lemma:unbiasedness_difference_estimator} in \cref{app:proof_diff_estimator} for a proof.
\end{proof}

Notably, \cref{ass:iid_sampling} is satisfied by the training procedure we described in \cref{sec:langmodel}, and is commonly true in machine learning.
However, it might not hold in general, as the train and validation distributions may not match exactly.
For example, NLP practitioners might deduplicate their training set but not validation \citep{biderman-etal-2023-pythia} or might use challenge sets for validation \citep{kiela-etal-2021-dynabench}.
Moreover, even when \cref{ass:iid_sampling} holds, \cref{eq:difference estimator} is low-variance only if we have large enough samples to compute $\yestimator_\checkpointtime(\cohorttime)$ and $\yestimator_\checkpointtime(\notseen)$.
Unfortunately, given the size of state-of-the-art LMs and their datasets, it can be expensive---both in terms of computation and memory usage---to extract the performance measures $Y_{\checkpointtime}(\xdatainstance)$ for all instances $\xdatainstance$ and checkpoints $\checkpointtime$.
Even with unlimited compute, $\yestimator_\checkpointtime(\cohorttime)$ can only be estimated using instances in $\batch_\cohorttime$, which lower bounds the variance of this estimator.

While the difference estimator is a first step towards a principled estimator of counterfactual memorisation, we can do better.
In the next section, we describe an estimator that has lower variance and requires weaker assumptions.

\vspace{-2pt}
\subsection{The Difference-in-Differences Estimator}\label{sec:estimation:did}
\vspace{-1pt}

We now introduce another causal estimator based on the difference-in-differences (\did) design.
The intuition behind \did is to use the time dimension to help with identification; \did identifies a causal estimand using the difference in the \emph{trends} over time of the outcome on treated vs.\ untreated instances.
In our specific setting, \did relies on the assumption that changes in model performance over time would follow similar trends in different batches if they had not been used for training.
We formalise this assumption as follows.

\begin{assumption}[Parallel Trends]\label{ass:parallel_trends}
    In the absence of training, the expected change in model performance across checkpoints would be the same regardless of treatment.
    That is, for all $\checkpointtime, \checkpointtimesec \geq \cohorttime \mathop{-} 1$:
    \begin{align}
         & \expect_{\xdatainstance}\big[Y_{\checkpointtime}(\xdatainstance; \notseen) - Y_{\checkpointtimesec}(\xdatainstance;\notseen) \mid \Cohorttime(\xdatainstance) \mathop{=} \cohorttime \big]                          \\
         & \qquad = \expect_{\xdatainstance}\big[Y_{\checkpointtime}(\xdatainstance; \notseen) - Y_{\checkpointtimesec}(\xdatainstance;\notseen) \mid \mathop{\Cohorttime(\xdatainstance) \mathop{=} \notseen} \big] \nonumber
    \end{align}
\end{assumption}

We need a second assumption before we can apply the \did design to our setting.\looseness=-1

\begin{assumption}[No Anticipation]\label{ass:no_anticipation}
    Training has no effect before it happens.
    That is, for all $\checkpointtime<\cohorttime$:
    \begin{align}
         & \expect_{\xdatainstance}\big[Y_{\checkpointtime}(\xdatainstance; \cohorttime) \mid \Cohorttime(\xdatainstance) \mathop{=} \cohorttime\big]                          \\
         & \qquad\qquad\quad =\expect_{\xdatainstance}\big[Y_{\checkpointtime} (\xdatainstance; \notseen)\mid\Cohorttime(\xdatainstance) \mathop{=} \cohorttime\big] \nonumber
    \end{align}
\end{assumption}

Given these two assumptions, we can now follow \citet{callaway-santanna-2021-difference} in identifying our target statistical estimand.\footnote{The \did design was originally proposed for the case with only two treatment and checkpoint steps (i.e., $\cohorttime \in \{1, \notseen\}$ and $\checkpointtime \in \{0, 1\}$). Previous work has shown the challenges of extending \did to multiple timesteps, especially when allowing for heterogeneous treatment effects \citep{roth-etal-2023-whats}. \citet{callaway-santanna-2021-difference} propose an extension which identifies \cref{eq:expect_memorisation} while allowing for treatment effect heterogeneity across checkpoint and treatment steps.}
Combined, these assumptions allow us to rewrite expectation \circled{2} in \cref{eq:expect_memorisation_breakdown} as a function of potential outcomes that are observable: $Y_{\checkpointtime}(\xdatainstance; \notseen)$, $Y_{\lagcohorttime}(\xdatainstance;\notseen)$ given $\Cohorttime(\xdatainstance) \mathop{=} \notseen$ are observable on a held-out validation set, while
$Y_{\lagcohorttime}(\xdatainstance;\cohorttime)$ given $\Cohorttime(\xdatainstance) \mathop{=} \cohorttime$ is observable on the training set.
The following statistical estimand thus identifies $\memorisationcohort$:
\begin{align}\label{eq:did_identification}
     & \didestimand = \expect_{\xdatainstance} [Y_{\checkpointtime}(\xdatainstance; \cohorttime) \!-\! Y_{\lagcohorttime}(\xdatainstance;\cohorttime) \!\mid \Cohorttime(\xdatainstance)\mathop{=}\cohorttime] \! \\
     & \qquad\,\,\, -\expect_{\xdatainstance} [Y_{\checkpointtime}(\xdatainstance; \notseen) \!-\! Y_{\lagcohorttime}(\xdatainstance; \notseen) \!\mid \Cohorttime(\xdatainstance)\mathop{=}\notseen]\nonumber
\end{align}
(See \cref{lemma:identification_did_estimand} in \cref{app:proof_did_estimand} for a proof.)
This leads to the following \did estimator.

\begin{estimator}\label{defn:didestimator}
    The \defn{difference-in-differences estimator} (\did), defined as:
    \begin{align}\label{eq:did_estimator}
        \didestimator = \underbrace{\!\Big(\yestimator_{\!\!\checkpointtime}(\cohorttime) \mathop{-} \yestimator_{\!\!\lagcohorttime}(\cohorttime)\Big)\!}_{\emph{diff in trained}} - \underbrace{\!\Big(\yestimator_{\!\!\checkpointtime}(\notseen) \mathop{-} \yestimator_{\!\!\lagcohorttime}(\notseen)\Big)\!}_{\emph{diff in untrained}}
    \end{align}
    is an unbiased estimator of $\memorisationcohort$ under \cref{ass:parallel_trends,ass:no_anticipation}.
\end{estimator}
\begin{proof}
    \vspace{-2pt}
    See \cref{lemma:unbiasedness_did_estimator} in \cref{app:proof_did_estimator} for a proof.
\end{proof}

The \did estimator depends on weaker assumptions and has a lower variance (under mild assumptions, see \cref{app:variance}) than the difference estimator in \cref{eq:difference estimator}.
Specifically, the parallel trends assumption (\cref{ass:parallel_trends}) is strictly weaker than the i.i.d.\ one (\cref{ass:iid_sampling}): if $p(\vx \mathop{\mid} \Cohorttime(\vx) \mathop{=}\cohorttime) = p(\vx \mathop{\mid} \Cohorttime(\vx) \mathop{=}\notseen)$, then it is trivial that performances should present parallel trends.
Moreover, \cref{ass:parallel_trends} only requires that the training and validation sets follow similar trends on average, which might be true even in the case of, e.g., challenge validation sets or deduplicated training data.
Therefore, the assumptions underpinning \did are more likely to hold in practice and we will use it to estimate memorisation here.

In practice, the difference-in-differences estimation procedure includes two steps.
First, we compute the model's performance on a subset of analysed instances---i.e., samples from the training and validation sets---using the available checkpoints; thus forming a \defn{panel} of observed outcomes, as it is usually termed in econometrics.
Then, we use this panel to compute the estimates in \cref{eq:did_estimator}.

\section{Prior Notions of Memorisation}\label{sec:related_work}

Memorisation has recently received much attention \citep[\textit{inter alia}]{arpit2017closer,carlinisecret,carlini-etal-2021-extracting,anagnostidis2023the}.\footnote{See \citet{hartmann-etal-2023-sok} or \citet{ishihara-2023-training} for surveys.\looseness=-1}
Prior work has studied how model architecture and training choices influence memorisation \citep{tirumala-etal-2022-memorization, kandpal2022deduplicating, lee-etal-2022-deduplicating,biderman2023emergent}, and where memorised instances are stored within a model \citep{maini2023neural,stoehr2024localizing}.
In this section, we use our framework to discuss three prior notions of memorisation which we consider most relevant to our paper: previous operationalisations of counterfactual memorisation \citep[e.g.,][]{feldman-2020-does}, influence functions \citep{zheng-jiang-2022-empirical}, and extractable memorisation \citep{carlini2023quantifying}.

\subsection{Previous Operationalisations of Counterfactual Memorisation}

As mentioned before, estimating an instance's memorisation $\memorisationinstance$ is non-trivial due to the counterfactual component in its definition.
We avoid this issue by estimating expected memorisation $\memorisationcohort$ instead.
Prior work \citep{feldman-2020-does, feldman-zhang-2020-what, zhang-etal-2023-counterfactual,lukasik2023larger} takes a different approach, comparing the performance of models trained with and without that instance.
In doing so, they average performance across training runs, measuring what we term \defn{architectural counterfactual memorisation}.

Formally, let $\trainingvar$ be a vector of variables responsible for training variance.
This includes, e.g., the data permutation induced by $\permutationfn$ and the initial model parameters $\vtheta_0$.
By defining a distribution $p(\trainingvar)$ over these variables, architectural memorisation can be defined as follows.

\begin{definition}\label{defn:arch_memorisation}
    \defn{Architectural counterfactual memorisation} is the counterfactual memorisation $\memorisationinstancelastcheckpoint$ when marginalising over training variables $\trainingvar$:
    \begin{align}\label{eq:arch_memorisation}
        \memarchitecture
         & \defeq \expect_{\trainingvar}\left[\memorisationinstancelastcheckpoint  \mathop{\mid} \Cohorttime(\xdatainstance) \mathop{\neq} \notseen \right]                                                                                                             \\
         & = \expect_{\trainingvar}\left[Y_{\lastcheckpointtimesmall}(\xdatainstance; \Cohorttime(\xdatainstance)) \mathop{-} Y_{\lastcheckpointtimesmall}(\xdatainstance; \notseen) \mathop{\mid} \Cohorttime(\xdatainstance) \mathop{\neq} \notseen \right] \nonumber
    \end{align}
    where $\Cohorttime(\xdatainstance)$ in the first potential outcome depends on which batch the shuffling function $\permutationfn$ puts $\xdatainstance$ in.\looseness=-1
\end{definition}

Prior work has proposed a number of methods to estimate this value \citep{bachmann2022generalization, lin-etal-2022-measuring, ilyas-etal-2022-datamodels, park-etal-2023-trak}.
The simplest of these is to train several models while including or not $\xdatainstance$ in the training set; these models are then used to approximate the expectation above.
We describe a statistical estimand and estimator for $\memarchitecture$ in \cref{app:arch_estimator}, discussing the assumptions needed by this approach.

Notably, this operationalisation has the advantage of estimating memorisation at the instance level.
However, it also has drawbacks---beyond just being computationally expensive to estimate.
These become apparent upon closer inspection of the definition of $\memarchitecture$.
First, it does not provide insights into the effect of the checkpoint step or treatment step on memorisation;
this is because $\lastcheckpointtime$ is hard-coded into $\memarchitecture$'s definition and because it marginalises over permutations of the data.
While it is trivial to generalise this definition to other checkpoint steps $\checkpointtime$ or to specific $\cohorttime$, prior work has mainly focused on these choices, overlooking the impact of training dynamics on memorisation.
Further, and perhaps more importantly, marginalising over $p(\trainingvar)$ means that this metric quantifies memorisation for a model architecture, rather than for a specific model.

\subsection{Influence Functions}\label{sec:influence_functions}

Influence functions \citep{hampel1974influence,cook1980characterizations} estimate---without re-training a model---how its parameters would change if an instance $\xdatainstance$ was removed from the training set.
Specifically, the new set of parameters can be approximated  as follows \citep{koh-liang-2017-understanding}:
\begin{align}\label{eq:influence_function}
    \optvthetawithoutx \approx \optvthetawithx + \nicefrac{1}{N} \, \hessian^{\mathop{-1}}_{\vtheta} \, \nabla_{\vtheta} \loss(\optvthetawithx, \xdatainstance)
\end{align}
where $\hessian_{\vtheta}$ is the hessian of $\loss$ evaluated at $\optvthetawithx$.
This approximation is based on a first-order Taylor expansion of the training objective around $\optvthetawithx$, and should lead to small errors under the following assumptions:
(i) the loss function is strictly convex in $\vtheta$, (ii) $\hessian_{\vtheta}$ is a positive-definite matrix, and (iii) the model has converged (i.e., the gradient is zero).
Given these assumptions, influence functions can be used to efficiently estimate the counterfactual term $Y_{\lastcheckpointtimesmall}(\xdatainstance; \notseen)$ in the definition of $\memorisationinstancelastcheckpoint$.
Specifically, we can measure the model's performance using the updated parameters, $\Ywithoutx(\xdatainstance) \defeq \perffn(\optvthetawithoutx; \xdatainstance)$, and equate $Y_{\lastcheckpointtimesmall}(\xdatainstance; \notseen) = \Ywithoutx(\xdatainstance)$.
The influence function estimator of memorisation can then be written as:
\begin{align}
    \memorisationinfluenceestimator = Y_{\lastcheckpointtimesmall}(\xdatainstance) - \Ywithoutx(\xdatainstance)
\end{align}
We formalise this estimator and its statistical estimand in \cref{app:influence_estimator}.
Notably, \citet{zheng-jiang-2022-empirical} use a similar approach to estimate memorisation in a classification setting.

Influence functions thus provide a computationally efficient method to estimate instance-level counterfactual memorisation.
However, none of the required assumptions above is typically satisfied for LMs, which can lead to strong biases in this estimator \citep{basu-etal-2020-influence, bae-etal-2022-if, schioppa-etal-2023-theoretical}.
Moreover, assumptions (ii) and (iii) require $\optvthetawithx$ to be locally optimal, meaning that this approach is not applicable for studying $\checkpointtime < \lastcheckpointtime$.
We therefore cannot use it to study how memorisation interacts with training dynamics.

\subsection{Extractable Memorisation}\label{sec:extractable_memorisation}

\citet{carlini2023quantifying} defines memorisation as \kl-extractability; a string is \kl-extractable if the model correctly predicts $\ell$ of its tokens given a prefix of $k$ tokens.
This definition has recently gained much interest because of its relevance to copyright infringement and data protection.
Despite being seemingly different, we argue that extractable memorisation is an estimator for counterfactual memorisation.
Concretely, let performance $\perffn$ be measured as whether a string is \kl-extractable.
Now, assume that a string is not \kl-extractable in the absence of training, i.e., $Y_{\checkpointtime}(\xdatainstance; \notseen) \mathop{=} 0$.
Given this assumption, we can define an extractable memorisation estimator as:
\begin{align}
    \klestimator = Y_{\checkpointtime}(\xdatainstance)
\end{align}
We formalise this estimator and its statistical estimand in \cref{app:extractable_estimator}.

Extractable memorisation thus implicitly assumes that $Y_{\checkpointtime}(\xdatainstance; \notseen) \mathop{=} 0$.
Counterfactual memorisation, on the other hand, controls for this counterfactual.
Notably, assuming $Y_{\checkpointtime}(\xdatainstance; \notseen) \mathop{=} 0$ may be reasonable when a string is long and complex; in this case, the chance that $\xdatainstance$ would be in the model’s top-\integer{1} beam (the output of greedy decoding) approaches zero.
However, this assumption may not be reasonable for shorter or less complex sequences. Rather, it may cause the resulting estimate to conflate memorisation with the intrinsic difficulty of predicting a string.

\section{Experiments}\label{sec:experiments}

While our method applies to any model trained with a single pass on its training data, we focus on quantifying memorisation in pretrained LMs, which are characterised by the use of architectures with a large number of parameters and large datasets.
Due to the costs of training such models from scratch, we take advantage of open-source pretrained models whose intermediate checkpoints and preprocessed data are publicly available.
In this section, we detail the models and data used and describe how we collect the observed outcomes $Y_{\checkpointtime}(\xdatainstance)$.

\paragraph{The Pythia Suite.}
We use the publicly available Pythia model suite\footnote{Both preprocessed data and intermediate checkpoints are publicly available at \href{https://github.com/EleutherAI/pythia}{\makesf{github.com/EleutherAI/pythia}}.} \citep{biderman-etal-2023-pythia}, composed of \integer{8} transformers of sizes ranging from \q{70}{\million} to \q{12}{\billion} parameters.
These models were trained on the Pile dataset \citep{gao-etal-2020-pile, biderman-etal-2022-datasheet}, a \q{300}{\billion}-token curated collection of English documents.
All models are trained using the same data.
Specifically, the dataset is shuffled and \enquote{packed} into sequences of \integer{2049} tokens;\footnote{Since target tokens are the right-shifted input tokens, to compute the loss on \integer{2048} tokens the Pythia authors added a token to the context.}
each of these sequences corresponds to an instance $\xdatainstance$.
Training was performed using a cosine learning rate schedule with warm-up, and using a batch size of \integer{1024} sequences, resulting in exactly \q{143}{\thousand} optimisation steps.
We use the model versions trained on the deduplicated Pile dataset to reduce the risk of finding spurious memorisation patterns due to duplication.
The deduplicated dataset has \q{207}{\billion} tokens, thus models using this version are trained for circa \float[1]{1.5} epochs to keep an equal token count relative to the non-deduplicated versions.
We consider the checkpoints relative to the first epoch (i.e., up to step \q{95}{\thousand}).\footnote{For completeness, we report the second half-epoch (steps \q{96}{\thousand}-\q{143}{\thousand}) analysis in \cref{app:additional_plots}.}
More details are in \cref{app:implementation_details}.

\paragraph{Constructing the Panel.}
Ideally, we would collect performance metrics for each instance $\xdatainstance$ and for every checkpoint step $\checkpointtime$.
However, given the size of the Pile dataset, it is computationally infeasible to collect evaluations for all instances; thus we resort to subsampling this data.
Furthermore, the granularity of the available checkpoints (i.e., every \q{1}{\thousand} timesteps) does not allow us to consider each timestep; thus we consider timesteps $\checkpointtime \in \{0, \q{1}{\thousand} \smalldots, \q{95}{\thousand}\}$ and treatment timesteps $\cohorttime \in \{\q{1}{\thousand}, \q{2}{\thousand} \smalldots, \q{95}{\thousand}\}$.
To match the checkpoint frequency, we consider all instances between two checkpoints (i.e., \q{1}{\thousand} batches) as if they were seen by the model at the same timestep.
We term these groups of batches \defn{\cohorts}\footnote{In econometrics group of instances that undergo treatment at the same time are typically termed \textit{cohorts}.} and formally define them as $\macrobatch_{\cohorttime} = \bigcup_{\cohorttime - \q{1}{\thousand} < t \leq \cohorttime} \batch_{t}$.
To obtain enough evaluations for each \cohort, we sample instances from the training set in two steps: we randomly choose \integer{10} batches for each \cohort and sample \integer{10} instances from each.
This process results in \q{14.3}{\thousand} analysed training instances.
Additionally, we sample \q{2}{\thousand} instances from the validation set to create $\macrobatch_{\notseen}$.
This process returns a panel of \q{16.3}{\thousand} instances evaluated at \integer{96} timesteps.\footnote{Our data and experimental artefacts are publicly available at \href{https://huggingface.co/collections/pietrolesci/memorisation-profiles-6619604c4594c878cd9d451f}{\makesf{huggingface.co/collections/pietrolesci/memorisation-profiles}}.}
As our performance metric we use the sequence-level log-likelihood: $\perffn(\vtheta, \vx) = \log \ptheta(\vx)$. \footnote{Results using different metrics are reported in \Cref{app:additional_plots}.}

\begin{figure*}[!t]
    \centering\small
    \includegraphics[width=\linewidth]{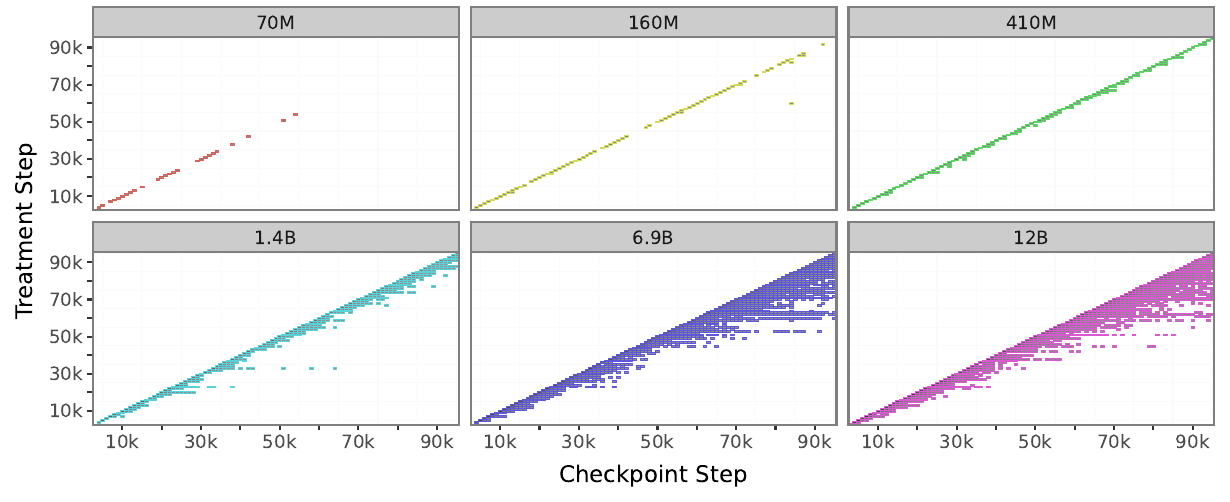}
    \vspace{-15pt}
    \caption{Memorisation profiles ($\memorisationcohort$). We only show statistically significant entries.}
    \label{fig:memorisation_profiles}
\end{figure*}

\paragraph{Statistical Inference.}
To compute statistical significance, we use the Simple Multiplier Bootstrap procedure of \citet{callaway-santanna-2021-difference} which returns simultaneous confidence intervals for all memorisation estimates, accounting for dependencies across \cohorts and checkpoint steps and thus avoiding multiple-testing issues.

\section{Results}

We report the memorisation profiles of all Pythia sizes in \cref{fig:memorisation_profiles}. Below, we use these memorisation profiles to describe different types of memorisation.

\begin{figure}[!t]
    \centering\small
    \includegraphics[width=\columnwidth]{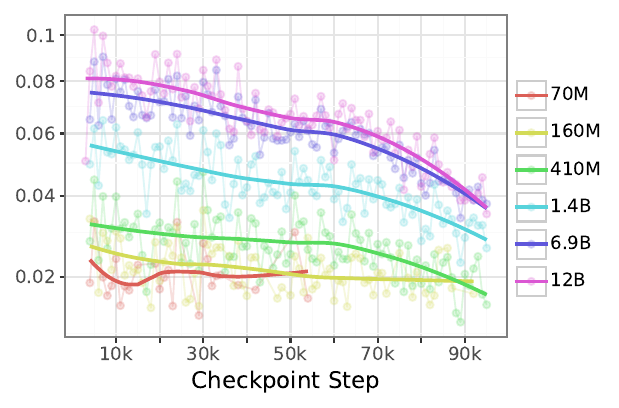}
    \vspace{-15pt}
    \caption{Instantaneous memorisation ($\memorisationcohort$ for $\cohorttime=\checkpointtime$). We only show statistically significant estimates.}
    \label{fig:instantaneous_memorisation}
\end{figure}

\paragraph{Instantaneous Memorisation.}
Instantaneous memorisation estimates (defined in \cref{sec:mem_definitions} as $\memorisationcohort$ when $\cohorttime \mathop{=} \checkpointtime$) are depicted as the diagonal entries in the memorisation profiles in \cref{fig:memorisation_profiles}, and are also presented in \cref{fig:instantaneous_memorisation}.
From these estimates, we can clearly observe the effect of the treatment step on memorisation: instantaneous memorisation is stronger earlier in training than later.
Interestingly (but perhaps unsurprisingly), instantaneous memorisation correlates with the cosine learning rate schedule: it is stronger after the warm-up period (around timestep \q{1.5}{\thousand}) than before it.
Furthermore, and as expected, instantaneous memorisation increases with model size.\footnote{Notably, we expect instantaneous memorisation to always be present in normally-trained LMs (albeit with a potentially small value). It could thus be used for power analysis \citep{cohen-1992-statistical}: choosing the number of instances to sample per \cohort which provides sufficient statistical power to correctly detect memorisation.}

\paragraph{Persistent Memorisation.}
Persistent memorisation estimates (defined in \cref{sec:mem_definitions} as $\memorisationcohort$ when $\cohorttime > \checkpointtime$) are depicted as the off-diagonal entries in the memorisation profiles in \cref{fig:memorisation_profiles}. \cref{fig:avg_persistent_memorisation} shows the average persistent memorisation at a specific number of timesteps after treatment; in this figure, $\memorisationcohort$ were averaged across \cohorts for each $\checkpointtime\mathop{-}\cohorttime$.\footnote{By averaging across \cohorts, variance is lower and more estimates become statistically significant.}
This way of aggregating the memorisation profile allows us to summarise the general memorisation patterns of a model.
Smaller models have lower persistent memorisation, with \q{70}{\million} having no persistent memorisation.
Interestingly, persistent memorisation plateaus after \q{25}{\thousand} timesteps.
This result has implications for data ordering during training.
For example, if there are particular instances that we do not want the model to memorise, but which we still want to use during training, they should be included in earlier batches.\footnote{We note that our results differ from \citeposs{biderman-etal-2023-pythia}, who find no differences in memorisation due to an instance's treatment step. We hypothesise that this discrepancy stems from the differences in metrics used to quantify memorisation and the statistical approaches adopted.}

\begin{figure}[!t]
    \centering\small
    \includegraphics[width=\columnwidth]{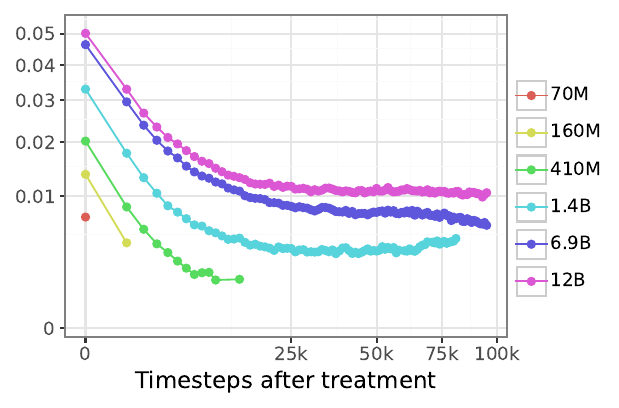}
    \vspace{-15pt}
    \caption{Average persistent memorisation ($\memorisationcohort$ averaged per timestep after treatment, i.e., $\checkpointtime\mathop{-}\cohorttime$). We only show statistically significant estimates.}
    \label{fig:avg_persistent_memorisation}
\end{figure}

\begin{figure*}[!t]
    \centering\small
    \includegraphics[width=\linewidth]{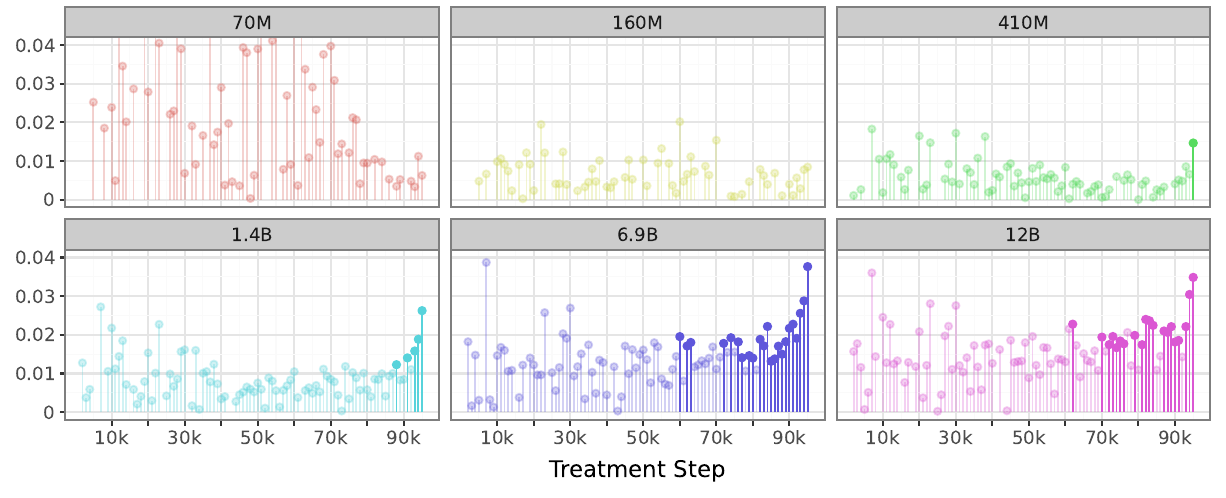}
    \vspace{-15pt}
    \caption{Residual memorisation ($\memorisationcohort$ for $\checkpointtime = \lastcheckpointtime = \q{95}{\thousand}$). Stronger colour intensity indicates statistical significance.}
    \label{fig:final_memorisation}
\end{figure*}

\begin{figure}[!t]
    \centering\small
    \includegraphics[width=\columnwidth]{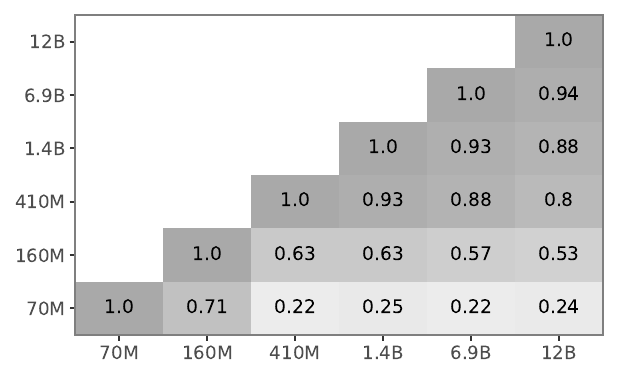}
    \vspace{-15pt}
    \caption{Pearson correlation between the memorisation profile of different models.}
    \label{fig:correlation_across_sizes}
\end{figure}

\paragraph{Residual Memorisation.}
Residual memorisation estimates (defined in \cref{sec:mem_definitions} as $\memorisationcohort$ when $\checkpointtime = \lastcheckpointtime$) are depicted as the final-column entries in \cref{fig:memorisation_profiles}, and are also presented in \cref{fig:final_memorisation} (we consider $\lastcheckpointtime$ to be the end of the first epoch here, i.e., timestep \q{95}{\thousand}).
Interestingly, while all \cohorts undergo some degree of instantaneous memorisation, it appears that many are then forgotten by the end of the first epoch, as shown by the statistically insignificant residual memorisation estimates.
Furthermore, in line with our persistent memorisation results, residual memorisation shows a recency effect: the final \cohorts are the most memorised.
We hypothesise that this recency effect can be explained by the learning rate schedule.
Specifically, when the learning rate is high, the optimisation process moves model parameters further in the locally optimal direction, thus \enquote{overwriting} previous information with new information; this results in higher instantaneous and lower residual memorisation.
On the contrary, towards the end of the training process, when the learning rate is lower, previous information is \enquote{forgotten} less as the updates are smaller (in expectation), resulting in higher residual and lower instantaneous memorisation.

\paragraph{Memorisation Across Scales.}
Due to the cost of training large LMs, it is highly desirable to be able to make predictions about a trained model's characteristics before undertaking training. One strategy is to derive insights from smaller models to inform the design of larger ones \citep{rae2021scaling,black2022gpt,scao2022language}.\footnote{Scaling laws for other notions of memorisation (\cref{sec:related_work}) have been studied in \citet{biderman2023emergent}.}
Predictability across scales is visually apparent in \cref{fig:instantaneous_memorisation,fig:avg_persistent_memorisation} where there are similar trends across model sizes.
We formalise this intuition in \cref{fig:correlation_across_sizes}, where we report the Pearson correlation between the memorisation profiles of different models.
Interestingly, memorisation for larger models (e.g., \q{12}{\billion}) is predictable from smaller ones (e.g., \q{410}{\million}).
We note that \q{70}{\million} and \q{160}{\million} are less predictive of the memorisation in \q{12}{\billion}.
However, prior work has shown that both these models suffer from training instability \citep{godey-etal-2024-why}; the reduction in predictability with these smaller models might thus be specific to the Pythia suite.

\section{Conclusions}

The memorisation of training data by neural networks has critical implications for privacy, copyright, and security. Thus, well-founded quantifications of memorisation, and corresponding accurate and efficient methods for their estimation are of great importance.
This work presents one such quantification and builds on the econometrics literature to derive an unbiased and efficient estimator of memorisation based on the difference-in-differences design.
We use this estimator to study the memorisation profiles of the Pythia model suite and find that memorisation is stronger and more persistent in larger models, determined by data order and learning rate, and stable across model sizes.

\section*{Limitations}

This work estimates counterfactual memorisation in pretrained LMs.
Unfortunately, due to the costs associated with running large pretrained LMs---even in inference mode---we experimented with a limited number of models (the Pythia suite) trained in a single language (English).
Investigating whether other model architectures, training procedures, and natural languages result in similar memorisation profiles would be important to strengthen our conclusions.
Furthermore, when collecting the panel data needed to estimate memorisation, we subsampled the number of evaluated instances; this can significantly increase our estimators' variance.

\section*{Acknowledgements}
\setlength{\intextsep}{0pt}
\setlength{\columnsep}{8pt}
\begin{wrapfigure}{l}{0.45\columnwidth}
    \includegraphics[width=0.45\columnwidth]{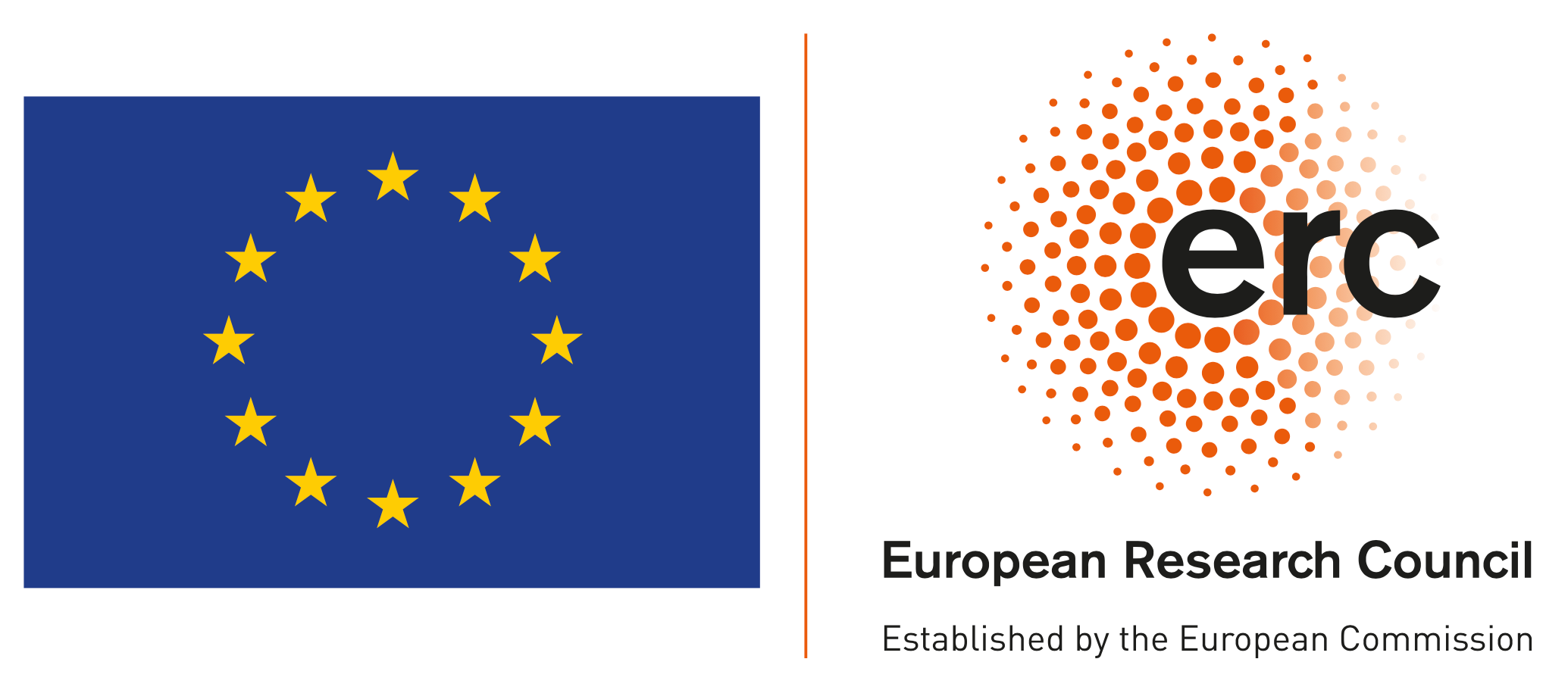}
\end{wrapfigure}
Pietro and Andreas received funding from the European Research Council (ERC) under the European Union’s Horizon 2020 Research and Innovation programme grant AVeriTeC (Grant agreement No. 865958).
Clara is funded by a Google PhD Fellowship.
We thank the anonymous reviewers, for their helpful questions and comments that helped us improve the paper.
We also thank Machel Reid for early discussions on the topic;
Sidak Pal Singh, Gregor Bachmann, and Ornella Darova for their feedback on earlier versions of this paper; and
Davide Lesci and Marco Lesci for proofreading the final version of the paper.

\bibliography{biblio.bib}

\clearpage
\appendix
\onecolumn

\section{Proofs}\label{app:proofs}

\subsection{Difference Estimand and Estimator}
\label{app:proof_diff_estimand}
\label{app:proof_diff_estimator}

\begin{lemma}[Identification of the Difference Estimand]
    The difference estimand, defined in \cref{eq:diff_identification}, identifies expected counterfactual memorisation (i.e., the causal estimand in \cref{eq:expect_memorisation}) under \cref{ass:iid_sampling}.
    \label{lemma:identification_difference_estimand}
\end{lemma}
\begin{proof}
    For this proof, we start with the definition of the difference estimand and, via algebraic manipulation, show its equality to expected counterfactual memorisation:
    \begin{subequations}
        \begin{align}
            \diffestimand
             & = \expect_{\xdatainstance} \big[Y_{\checkpointtime}(\xdatainstance; \cohorttime)  \mid \Cohorttime(\xdatainstance) = \cohorttime \big] - \expect_{\xdatainstance} \big[Y_{\checkpointtime}(\xdatainstance; \notseen) \mid \Cohorttime(\xdatainstance) = \notseen\big]                                                  \\
             & = \expect_{\xdatainstance} \big[Y_{\checkpointtime}(\xdatainstance; \cohorttime)  \mid \Cohorttime(\xdatainstance) = \cohorttime \big] - \expect_{\xdatainstance} \big[Y_{\checkpointtime}(\xdatainstance; \notseen) \mid \Cohorttime(\xdatainstance) = \cohorttime \big] & \mathcomment{By \cref{ass:iid_sampling}}   \\
             & = \expect_{\xdatainstance} \big[Y_{\checkpointtime}(\xdatainstance; \cohorttime) - Y_{\checkpointtime}(\xdatainstance; \notseen) \mid \Cohorttime(\xdatainstance) = \cohorttime \big]                                                                                     & \mathcomment{By linearity of expectations} \\
             & = \memorisationcohort
        \end{align}
    \end{subequations}
    This completes the proof.
\end{proof}

\begin{lemma}[Unbiasedness of the Difference Estimator]
    The difference estimator, defined in \cref{def:differenceestimator}, is an unbiased estimator of expected counterfactual memorisation $\memorisationcohort$ under \cref{ass:iid_sampling}.
    \label{lemma:unbiasedness_difference_estimator}
\end{lemma}
\begin{proof}
    To prove this estimator is unbiased, let us first define the probability of sampling a batch $\batch_\cohorttime$:
    \begin{align} \label{eq:proof_diff:prob_batch}
        p(\batch_\cohorttime) = \prod_{\xdatainstance \in \batch_\cohorttime} p(\xdatainstance \mid \Cohorttime(\xdatainstance) = \cohorttime)
    \end{align}
    Taking the expectation of our estimator with respect to the batches used for its estimation we see that:
    \begin{subequations}
        \begin{align}
            \expect_{\batch_\cohorttime, \batch_\notseen}\left[\diffestimator\right]
             & =  \expect_{\batch_\cohorttime, \batch_\notseen}\left[\yestimator_{\checkpointtime}(\cohorttime) - \yestimator_{\checkpointtime}(\notseen)\right]                                                                                                                                                               \\
             & =  \expect_{\batch_\cohorttime, \batch_\notseen}\left[\frac{1}{|\batch_\cohorttime|} \sum_{\xdatainstance \in \batch_\cohorttime} Y_{\checkpointtime}(\xdatainstance) - \frac{1}{|\batch_\notseen|} \sum_{\xdatainstance \in \batch_\notseen} Y_{\checkpointtime}(\xdatainstance)\right]                        \\
             & =  \expect_{\batch_\cohorttime}\left[\frac{1}{|\batch_\cohorttime|} \sum_{\xdatainstance \in \batch_\cohorttime} Y_{\checkpointtime}(\xdatainstance)\right]  -
            \expect_{\batch_\notseen}\left[\frac{1}{|\batch_\notseen|} \sum_{\xdatainstance \in \batch_\notseen} Y_{\checkpointtime}(\xdatainstance)\right] \label{eq:proof_diff:sep_expectations}                                                                                                                             \\
             & = \expect_{\xdatainstance} \big[Y_{\checkpointtime}(\xdatainstance; \cohorttime)  \mid \Cohorttime(\xdatainstance) = \cohorttime \big] - \expect_{\xdatainstance} \big[Y_{\checkpointtime}(\xdatainstance; \notseen) \mid \Cohorttime(\xdatainstance) = \notseen\big] \label{eq:proof_diff:avg_to_expectations} \\
             & = \diffestimand
        \end{align}
    \end{subequations}
    where \cref{eq:proof_diff:sep_expectations} follows because the sampling of $\batch_\cohorttime$ and $\batch_\notseen$ are independent;
    \cref{eq:proof_diff:avg_to_expectations} holds due to \cref{eq:proof_diff:prob_batch} and the unbiasedness of Monte Carlo estimators.
    We can now invoke \cref{lemma:identification_difference_estimand}, which states that $\diffestimand$ identifies $\memorisationcohort$ under the i.i.d.\ assumption (\cref{ass:iid_sampling}).
    Thus, we have that the expected value of our estimator is equal to $\memorisationcohort$, finishing the proof.
\end{proof}

\subsection{Difference-in-Differences Estimand and Estimator}
\label{app:proof_did_estimand}
\label{app:proof_did_estimator}

\begin{lemma}[Identification of the Difference-in-Differences Estimand]
    The \did estimand, defined in \cref{eq:did_identification}, identifies expected counterfactual memorisation (i.e., the causal estimand in \cref{eq:expect_memorisation}) under \cref{ass:parallel_trends,ass:no_anticipation} for all $\checkpointtime\geq\cohorttime$.\looseness=-1
    \label{lemma:identification_did_estimand}
\end{lemma}
\begin{proof}

    To prove this lemma, we first note that by the no anticipations assumption (\cref{ass:no_anticipation}):
    \begin{align}
        \expect_{\xdatainstance}\left[
            Y_{{\color{\checkpointcolor}\cohorttime-1}}(\xdatainstance;\notseen)
            \mid \Cohorttime(\xdatainstance) = \cohorttime\right] =
        \expect_{\xdatainstance}\left[
            Y_{{\color{\checkpointcolor}\cohorttime-1}}(\xdatainstance;\cohorttime)
            \mid \Cohorttime(\xdatainstance) = \cohorttime\right] \label{eq:no_anticipation_rem}
    \end{align}
    Furthermore, by the parallel trends assumption (\cref{ass:parallel_trends}) and linearity of expectations:
    \begin{align}
         & \expect_{\xdatainstance}\big[Y_{\checkpointtime}(\xdatainstance; \notseen) - Y_{\checkpointtimesec}(\xdatainstance;\notseen) \mid \Cohorttime(\xdatainstance) = \cohorttime \big]  = \expect_{\xdatainstance}\big[Y_{\checkpointtime}(\xdatainstance; \notseen) - Y_{\checkpointtimesec}(\xdatainstance;\notseen) \mid \Cohorttime(\xdatainstance) = \notseen \big] \label{eq:parallel_trend_rem}                                                                                                                           \\
         & \,\, \implies \expect_{\xdatainstance}\big[Y_{\checkpointtime}(\xdatainstance;\notseen)  \mathop{\mid} \Cohorttime(\xdatainstance) \mathop{=} \cohorttime \big] = \expect_{\xdatainstance}\big[Y_{\checkpointtimesec}(\xdatainstance; \notseen) \mathop{\mid} \Cohorttime(\xdatainstance) \mathop{=} \cohorttime \big]  - \expect_{\xdatainstance}\big[Y_{\checkpointtime}(\xdatainstance; \notseen) - Y_{\checkpointtimesec}(\xdatainstance;\notseen)  \mathop{\mid} \Cohorttime(\xdatainstance) \mathop{=} \notseen \big]
        \nonumber
    \end{align}
    As in \cref{lemma:identification_difference_estimand}, we now start with the definition of the DiD estimand and, via algebraic manipulation, show its equivalence to expected counterfactual memorisation:
    \begin{subequations}
        \begin{align}
             & \didestimand \nonumber                                                                                                                                                                                                                    \\
             & \quad = \expect \big[Y_{\checkpointtime}(\xdatainstance; \cohorttime) - Y_{{\color{\checkpointcolor}\cohorttime-1}}(\xdatainstance; \cohorttime) \mid \Cohorttime(\xdatainstance) = \cohorttime\big] -
            \expect \big[Y_{\checkpointtime}(\xdatainstance; \notseen) - Y_{{\color{\checkpointcolor}\cohorttime-1}}(\xdatainstance; \notseen) \mid \Cohorttime(\xdatainstance) = \notseen\big]                                                          \\
             & \quad = \expect \big[Y_{\checkpointtime}(\xdatainstance; \cohorttime)\!\mid\! \Cohorttime(\xdatainstance) \!=\! \cohorttime\big]
            - \underbrace{\expect \big[Y_{{\color{\checkpointcolor}\cohorttime-1}}(\xdatainstance; \cohorttime) \!\mid\! \Cohorttime(\xdatainstance) \!=\! \cohorttime\big]}_{\mathcomment{no anticipation}} - \expect \big[Y_{\checkpointtime}(\xdatainstance; \notseen)
            - Y_{{\color{\checkpointcolor}\cohorttime-1}}(\xdatainstance; \notseen) \!\mid\! \Cohorttime(\xdatainstance) \!=\! \notseen\big]  \label{eq:no_ant}                                                                                          \\
             & \quad =
            \expect \big[ Y_{\checkpointtime}(\xdatainstance; \cohorttime) \!\mid\! \Cohorttime(\xdatainstance) \!=\! \cohorttime\big] -
            \underbrace{
                \expect \big[Y_{{\color{\checkpointcolor}\cohorttime-1}}(\xdatainstance; \notseen) \!\mid\! \Cohorttime(\xdatainstance) \!=\! \cohorttime\big] -
                \expect \big[
                    Y_{\checkpointtime}(\xdatainstance; \notseen) -
                    Y_{{\color{\checkpointcolor}\cohorttime-1}}(\xdatainstance; \notseen) \!\mid\! \Cohorttime(\xdatainstance) \!=\! \notseen\big]}_{\mathcomment{parallel trends}}
            \label{eq:parallel_trends}                                                                                                                                                                                                                   \\
             & \quad = \expect \big[Y_{\checkpointtime}(\xdatainstance; \cohorttime)\mid \Cohorttime(\xdatainstance) = \cohorttime\big] - \expect \big[Y_{\checkpointtime}(\xdatainstance; \notseen) \mid \Cohorttime(\xdatainstance) = \cohorttime\big] \\
             & \quad = \expect \big[Y_{\checkpointtime}(\xdatainstance; \cohorttime) - Y_{\checkpointtime}(\xdatainstance; \notseen) \mid \Cohorttime(\xdatainstance) = \cohorttime\big] \label{eq:lemma_diff_estimand:linearity_of_expectaitons}        \\
             & \quad = \memorisationcohort
        \end{align}
    \end{subequations}
    where we replace the terms in \cref{eq:no_ant,eq:parallel_trends} using their equivalences given in \cref{eq:no_anticipation_rem} and \cref{eq:parallel_trend_rem}, respectively.
    This completes the proof.
    We note that a similar proof is available in Lemma A.1 in \citet{callaway-santanna-2021-difference}.
\end{proof}

\begin{lemma}[Unbiasedness of the Difference-in-Differences Estimator]
    The difference-in-differences estimator, defined in \cref{defn:didestimator}, is an unbiased estimator of expected counterfactual memorisation $\memorisationcohort$ under \cref{ass:parallel_trends,ass:no_anticipation}.
    \label{lemma:unbiasedness_did_estimator}
\end{lemma}
\begin{proof}
    We can follow the same logic as in \cref{lemma:unbiasedness_difference_estimator} because the same properties hold (i.e., the sampling of $\batch_\cohorttime$ and $\batch_\notseen$ are independent,
    the joint probability of a set is the product of the probability of sampling individual instances, and the unbiasedness of Monte Carlo estimators). We then arrive at the following equivalence:
    \begin{subequations}
        \begin{align}
            \expect_{\batch_\cohorttime, \batch_\notseen}\!\!\!\!\left[\didestimator\right]
             & = \expect_{\batch_\cohorttime, \batch_\notseen}\bigg[\yestimator_{\checkpointtime}(\cohorttime) - \yestimator_{{\color{\checkpointcolor}\cohorttime-1}}(\cohorttime) - \yestimator_{\checkpointtime}(\notseen) + \yestimator_{{\color{\checkpointcolor}\cohorttime-1}}(\notseen) \bigg]                                                                                                                                                                       \\
             & =  \expect_{\batch_\cohorttime, \batch_\notseen}\left[\frac{1}{|\batch_\cohorttime|} \sum_{\xdatainstance \in \batch_\cohorttime} \Big(Y_{\checkpointtime}(\xdatainstance) - Y_{{\color{\checkpointcolor}\cohorttime-1}}(\xdatainstance)\Big) - \frac{1}{|\batch_\notseen|} \sum_{\xdatainstance \in \batch_\notseen} \Big(Y_{\checkpointtime}(\xdatainstance)- Y_{{\color{\checkpointcolor}\cohorttime-1}}(\xdatainstance)\Big)\right]                       \\
             & =  \expect_{\batch_\cohorttime}\left[\frac{1}{|\batch_\cohorttime|} \sum_{\xdatainstance \in \batch_\cohorttime} \Big(Y_{\checkpointtime}(\xdatainstance) - Y_{{\color{\checkpointcolor}\cohorttime-1}}(\xdatainstance)\Big)\right] - \expect_{ \batch_\notseen}\left[\frac{1}{|\batch_\notseen|} \sum_{\xdatainstance \in \batch_\notseen} \Big(Y_{\checkpointtime}(\xdatainstance)- Y_{{\color{\checkpointcolor}\cohorttime-1}}(\xdatainstance)\Big)\right] \\
             & = \expect_{\xdatainstance} \big[Y_{\checkpointtime}(\xdatainstance; \cohorttime) - Y_{{\color{\checkpointcolor}\cohorttime-1}}(\xdatainstance; \cohorttime) \mathop{\mid} \Cohorttime(\xdatainstance)\! = \!\cohorttime\big] - \expect_{\xdatainstance} \big[Y_{\checkpointtime}(\xdatainstance; \notseen) - Y_{{\color{\checkpointcolor}\cohorttime-1}}(\xdatainstance; \notseen) \mathop{\mid} \Cohorttime(\xdatainstance)\! = \!\notseen\big]              \\
             & = \didestimand
        \end{align}
    \end{subequations}
    Finally, we invoke \cref{lemma:identification_did_estimand} which proves that $\didestimand$ identifies $\memorisationcohort$ under \cref{ass:parallel_trends,ass:no_anticipation}.
\end{proof}

\subsection{Variances of Estimators}\label{app:variance}

\newcommand{\var}{\mathrm{Var}}
\newcommand{\cov}{\mathrm{Cov}}
We assume that all potential outcomes $Y_{\checkpointtime}(\xdatainstance;\cohorttime)$ have the same variance $\sigma^2$.
We now first look at the variance of the difference estimator.
To this end, let's consider the variance of $\yestimator_\checkpointtime(\cohorttime)$:
\begin{align}
    \var\Big(\yestimator_\checkpointtime(\cohorttime) \Big)
     & = \var\left(\frac{1}{|\batch_{\cohorttime}|}\sum_{\vx \in \batch_\cohorttime}Y_{\checkpointtime}(\xdatainstance) \right)
    = \frac{1}{|\batch_{\cohorttime}|^2} \sum_{\vx \in \batch_\cohorttime} \var\Big(Y_{\checkpointtime}(\xdatainstance; \cohorttime) \mid \Cohorttime(\xdatainstance) = \cohorttime \Big)
    = \frac{|\batch_{\cohorttime}|\,\sigma^2}{|\batch_{\cohorttime}|^2}
    = \frac{\sigma^2}{|\batch_{\cohorttime}|}
\end{align}
This is simply the variance of estimating an expectation using the mean of $|\batch_{\cohorttime}|$ i.i.d. random variables, each with variance $\sigma^2$. We can similarly derive the variance of $\yestimator_\checkpointtime(\notseen)$.
The variance of $\diffestimator$ is then:%
\begin{equation}
    \var(\diffestimator) = \frac{\sigma^2}{|\batch_{\cohorttime}|} + \frac{\sigma^2}{|\batch_{\notseen}|} - 2\,\cov(\yestimator_\checkpointtime(\cohorttime), \yestimator_\checkpointtime(\notseen))
\end{equation}
Assuming batches $\batch_{\cohorttime}$ and $\batch_{\notseen}$ were drawn independently,
then the estimators $\yestimator_\checkpointtime(\cohorttime)$ and $\yestimator_\checkpointtime(\notseen)$ should also be independent.
Thus, $\cov(\yestimator_\checkpointtime(\cohorttime), \yestimator_\checkpointtime(\notseen)) = 0$.

We now look at the variance of the difference-in-differences estimator.
Let the correlation between $Y_{\checkpointtime}(\xdatainstance;\cohorttime)$ and $Y_{{\color{\checkpointcolor}\cohorttime-1}}(\xdatainstance;\cohorttime)$ be $\rho$.
These are, respectively, the potential outcomes of our model on a specific instance $\xdatainstance$ before and after training on it.
For shorthand, let $\Delta \yestimator_\cohorttime = \yestimator_\checkpointtime(\cohorttime) - \yestimator_{\color{\checkpointcolor}\cohorttime-1}(\cohorttime)$ and $\Delta \yestimator_\notseen = \yestimator_\checkpointtime(\notseen) - \yestimator_{\color{\checkpointcolor}\cohorttime-1}(\notseen)$.
We can show that:
\begin{subequations}
    \begin{align}
        \var(\Delta \yestimator_\cohorttime ) & =\var \Big(\yestimator_\checkpointtime(\cohorttime) - \yestimator_{\color{\checkpointcolor}\cohorttime-1}(\cohorttime)\Big)                                                                                                                              \\
                                              & = \var \left( \frac{1}{|\batch_{\cohorttime}|}\sum_{\xdatainstance \in \batch_\cohorttime}{\Big(Y_{\checkpointtime}(\xdatainstance;\cohorttime) - Y_{{\color{\checkpointcolor}\cohorttime-1}}(\xdatainstance;\cohorttime)\Big)} \right)                  \\
                                              & = \frac{1}{|\batch_{\cohorttime}|^2} \sum_{\xdatainstance \in \batch_\cohorttime}\bigg(\sigma^2+\sigma^2-2\,\cov\Big(Y_{\checkpointtime}(\xdatainstance;\cohorttime), Y_{{\color{\checkpointcolor}\cohorttime-1}}(\xdatainstance;\cohorttime)\Big)\bigg) \\
                                              & =\frac{1}{|\batch_{\cohorttime}|^2} \sum_{\xdatainstance \in \batch_\cohorttime}{(2\sigma^2-2\rho\sigma^2)} = \frac{2\sigma^2}{|\batch_{\cohorttime}|}(1-\rho)
    \end{align}
\end{subequations}
We can derive the variance for $\Delta \yestimator_\notseen$ in the exact same manner.
We thus have that:
\begin{align}
    \var(\didestimator) =  \var(\Delta \yestimator_\cohorttime) + \var(\Delta \yestimator_\notseen) - 2\cov(\Delta \yestimator_\cohorttime, \Delta \yestimator_\notseen)
\end{align}
Note that $\Delta \yestimator_\cohorttime$ and $\Delta \yestimator_\notseen$ are estimated with independent samples, and thus, $\cov(\Delta \yestimator_\cohorttime, \Delta \yestimator_\notseen) = 0$.
We can thus rewrite this estimator's variance as:
\begin{align}
    \var(\didestimator) =  \frac{2\sigma^2}{|\batch_{\cohorttime}|}(1-\rho_\cohorttime) + \frac{2\sigma^2}{|\batch_{\cohorttime}|}(1-\rho_\notseen)
\end{align}
If we have $\rho_\cohorttime > 0.5$ and $\rho_\notseen>0.5$, then the variance of $\didestimator$ should be lower than that of the $\diffestimator$.
This is a reasonable assumption since---for fixed timesteps ${\color{\checkpointcolor}\cohorttime\mathop{-}1}$ and $\checkpointtime$---there should be a strong relationship between a model's performance on an instance before (${\color{\checkpointcolor}\cohorttime\mathop{-}1}$) and after ($\checkpointtime$) it has been trained on due to factors such as vocabulary richness and grammatical structure.\looseness=-1

\section{Statistical Estimands and Estimators in Prior Work}

In this section, we formalise prior works' estimators of memorisation using our formalisation of counterfactual memorisation.

\subsection{Architectural Counterfactual Memorisation}\label{app:arch_estimator}

In this section, we describe one potential estimator for architectural counterfactual memorisation $\memarchitecture$ (in \cref{defn:arch_memorisation}).
First, we need the following assumption in order to identify the causal estimand for this quantity.\looseness=-1
\begin{assumption}[Negligible training effect] \label{ass:negligible_training_effect}
    In expectation, the effect of having a specific instance in the training set is negligible on any validation instance. That is, for any two instances $\xdatainstance$ and $\xdatainstance'$:
    \begin{align}
         & \expect_{\trainingvar}\left[Y_{\lastcheckpointtimesmall}(\xdatainstance'; \notseen) \mathop{\mid} \Cohorttime(\xdatainstance) \mathop{=} \notseen \right] = \expect_{\trainingvar}\left[Y_{\lastcheckpointtimesmall}(\xdatainstance'; \notseen) \mathop{\mid} \Cohorttime(\xdatainstance) \mathop{\neq} \notseen \right]
    \end{align}
\end{assumption}

Given this assumption, we can identify the following statistical estimand for $\memarchitecture$:
\begin{align}\label{eq:arch_memorisation_stat_estimand}
    \memorisationarchitecture
     & = \expect_{\trainingvar}\left[Y_{\lastcheckpointtimesmall}(\xdatainstance; \Cohorttime(\xdatainstance)) \mid \Cohorttime(\xdatainstance) \neq \notseen \right]
    - \expect_{\trainingvar}\left[Y_{\lastcheckpointtimesmall}(\xdatainstance; \notseen) \mid \Cohorttime(\xdatainstance) = \notseen \right]
\end{align}
We now define the architectural estimator, associated with this statistical estimand.
\begin{estimator}
    The \defn{architectural estimator}, defined as:\footnote{We note that prior work has proposed more efficient estimators of the above \citep{bachmann2022generalization, lin-etal-2022-measuring, ilyas-etal-2022-datamodels, park-etal-2023-trak}. However, these estimators remain computationally expensive for large LMs.}
    \begin{align}
        \memorisationarchitectureestimator
        = \frac{1}{|\paramswithx|} \sum_{\vtheta \in \paramswithx} Y_{\lastcheckpointtimesmall}(\xdatainstance)  - \frac{1}{|\paramswithoutx|} \sum_{\vtheta \in \paramswithoutx} Y_{\lastcheckpointtimesmall}(\xdatainstance)
    \end{align}
    is an unbiased estimator of $\memarchitecture$ under \cref{ass:negligible_training_effect}.
    In this equation, $\paramswithx$ and $\paramswithoutx$ are sets of model parameters trained independently with or without $\xdatainstance$ in the training set.
\end{estimator}
\begin{proof}
    First, we prove that the statistical estimand $\memorisationarchitecture $ identifies $\memarchitecture$:
    \begin{subequations}
        \begin{align}
            \memorisationarchitecture
             & = \expect_{\trainingvar}\left[Y_{\lastcheckpointtimesmall}(\xdatainstance; \Cohorttime(\xdatainstance)) \mid \Cohorttime(\xdatainstance) \neq \notseen \right]
            - \expect_{\trainingvar}\left[Y_{\lastcheckpointtimesmall}(\xdatainstance; \notseen) \mid \Cohorttime(\xdatainstance) = \notseen \right]                                                                                 \\
             & = \expect_{\trainingvar}\left[Y_{\lastcheckpointtimesmall}(\xdatainstance; \Cohorttime(\xdatainstance)) \mid \Cohorttime(\xdatainstance) \neq \notseen \right]  -
            \expect_{\trainingvar}\left[Y_{\lastcheckpointtimesmall}(\xdatainstance; \notseen)\mid \Cohorttime(\xdatainstance) \neq \notseen\right]
             & \mathcomment{By \cref{ass:negligible_training_effect}}                                                                                                                                                                \\
             & = \expect_{\trainingvar}\left[Y_{\lastcheckpointtimesmall}(\xdatainstance; \Cohorttime(\xdatainstance)) - Y_{\lastcheckpointtimesmall}(\xdatainstance; \notseen)\mid \Cohorttime(\xdatainstance) \neq \notseen\right]
             & \mathcomment{Linearity of expectations}                                                                                                                                                                               \\
             & = \memarchitecture
        \end{align}
    \end{subequations}
    We now prove the estimator above is unbiased:
    \begin{subequations}
        \begin{align}
            \memorisationarchitectureestimator
             & = \expect_{\paramswithx, \paramswithoutx}\left[ \frac{1}{|\paramswithx|} \sum_{\vtheta \in \paramswithx} Y_{\lastcheckpointtimesmall}(\xdatainstance)  - \frac{1}{|\paramswithoutx|} \sum_{\vtheta \in \paramswithoutx} Y_{\lastcheckpointtimesmall}(\xdatainstance) \right]
            \\
             & = \expect_{\paramswithx}\left[ \frac{1}{|\paramswithx|} \sum_{\vtheta \in \paramswithx} Y_{\lastcheckpointtimesmall}(\xdatainstance) \right]
            - \expect_{\paramswithoutx}\left[ \frac{1}{|\paramswithoutx|} \sum_{\vtheta \in \paramswithoutx} Y_{\lastcheckpointtimesmall}(\xdatainstance) \right]                                                                                                                           \\
             & = \expect_{\trainingvar}\left[ Y_{\lastcheckpointtimesmall}(\xdatainstance; \Cohorttime(\xdatainstance)) \mid \Cohorttime(\xdatainstance) \neq \notseen \right]
            - \expect_{\trainingvar}\left[ Y_{\lastcheckpointtimesmall}(\xdatainstance; \notseen) \mid \Cohorttime(\xdatainstance) = \notseen\right]                                                                                                                                        \\
             & = \memorisationarchitecture
        \end{align}
    \end{subequations}
    This completes the proof.
\end{proof}

\subsection{Influence Functions}\label{app:influence_estimator}

As mentioned in \cref{sec:influence_functions}, influence functions approximate $\optvthetawithoutx$ using a first-order Taylor expansion of the training objective around $\optvthetawithx$.
This should lead to small errors under the following assumptions:
(i) the loss function is strictly convex in $\vtheta$, (ii) $\hessian_{\vtheta}$ is a positive-definite matrix, and (iii) the model has converged \citep{koh-liang-2017-understanding}.
We make these assumptions explicit now.

\begin{assumption}[Strict Convexity] \label{ass:convexity}
    The loss function $\loss$ is strictly convex with respect to the parameters $\vtheta$.
\end{assumption}

\begin{assumption}[Local Optimality] \label{ass:optimality}
    The parameters $\optvthetawithx$ locally minimise the loss function $\loss$, meaning that the
    $\hessian_{\vtheta}$ is a positive-definite matrix and that gradient of the loss with respect to the parameters $\optvthetawithx$ is zero.
\end{assumption}

Given these assumptions, we can estimate the counterfactual term in $\memorisationinstance$ (in \cref{eq:instance_memorisation}) by computing the performance using the updated parameters $\optvthetawithoutx$.
As mentioned in the main text, we thus define $\Ywithoutx(\xdatainstance) = \perffn(\optvthetawithoutx; \xdatainstance)$ and equate $Y_{\lastcheckpointtimesmall}(\xdatainstance; \notseen) = \Ywithoutx(\xdatainstance)$.

\begin{estimator}
    The \defn{influence function estimator}, defined as:
    \begin{align}
        \memorisationinfluenceestimator = Y_{\lastcheckpointtimesmall}(\xdatainstance) - \Ywithoutx(\xdatainstance)
    \end{align}
    is an unbiased estimator of $\memorisationinstancelastcheckpoint$ under \cref{ass:convexity,ass:optimality}.
\end{estimator}
\vspace{-6pt}
\begin{proof}
    See \citet{cook1980characterizations} or \citet{koh-liang-2017-understanding} for derivations of how $\Ywithoutx(\xdatainstance)$ approximates the counterfactual $Y_{\lastcheckpointtimesmall}(\xdatainstance; \notseen)$ under the assumptions above. The estimator then follows trivially from replacing $Y_{\lastcheckpointtimesmall}(\xdatainstance; \notseen)$ in \cref{eq:instance_memorisation}.
\end{proof}

\subsection{Extractable Memorisation} \label{app:extractable_estimator}

As mentioned in \cref{sec:extractable_memorisation}, extractable memorisation assumes zero-valued counterfactual performances $Y_{\checkpointtime}(\xdatainstance; \notseen) = 0$.
We formalise this assumption, and the associated statistical estimand and estimator in this section.

\begin{assumption}[Negligible counterfactual] \label{ass:negligible_counterfactual}
    In the absence of training, performance on a string should be zero: $Y_{\checkpointtime}(\xdatainstance; \notseen) = 0$.
\end{assumption}

Given this assumption, we can trivially identify counterfactual memorisation as being equivalent to the statistical estimand: $\klestimand = Y_{\checkpointtime}(\xdatainstance; \cohorttime)$.
We can now define the \kl-extractable memorisation estimator under our framework.
\begin{estimator}
    The \defn{$\boldsymbol{(k,\ell)}$-extractable memorisation} estimator, defined as:
    \begin{align}
        \klestimator =
        Y_{\checkpointtime}(\xdatainstance)
    \end{align}
    is an unbiased estimator of $\memorisationinstance$ under \cref{ass:negligible_counterfactual}.
\end{estimator}
\vspace{-6pt}
\begin{proof}
    This follows trivially from replacing $Y_{\checkpointtime}(\xdatainstance; \notseen)$ with 0 in \cref{eq:instance_memorisation}.
\end{proof}

\section{Implementation Details}\label{app:implementation_details}

We implement all experiments using the \makesf{PyTorch} framework \citep{paszke-etal-2019-pytorch}.
We use the Pythia models as available through the \makesf{transformers} library \citep{wolf-etal-2020-transformers}.
For a consistent evaluation between scales, we load every model using \makesf{bfloat16} precision, which is needed for the larger versions. We control randomness using  \makesf{CUDA} deterministic operations and seeding the pseudo-random number generators at every level of the stack and for each multi-processing worker.
We use the implementation of the \citet{callaway-santanna-2021-difference} estimator as available in the \makesf{differences} library\footnote{\href{https://github.com/bernardodionisi/differences}{\makesf{github.com/bernardodionisi/differences}}.}.

\subsection{The Pythia Suite}

We use the publicly available Pythia model suite \citep{biderman-etal-2023-pythia}, which was trained on the Pile \citep{gao-etal-2020-pile, biderman-etal-2022-datasheet}. Both the preprocessed training data and intermediate checkpoints are publicly available.\footnote{\href{https://github.com/EleutherAI/pythia}{\makesf{github.com/EleutherAI/pythia}}.}

\paragraph{Data.}
The Pile is a \q{300}{\billion}-token curated collection of English documents.
The deduplicated version of the dataset is obtained by applying a near-deduplication method based on \makesf{MinHashLSH} and has \q{207}{\billion} tokens.
Before being used for training, the dataset is shuffled, tokenised, and \enquote{packed} into sequences of \integer{2049} tokens with no end-of-document token.\footnote{\href{https://github.com/EleutherAI/pythia/issues/123}{\makesf{github.com/EleutherAI/pythia/issues/123}}.}
By design, each sequence can pack multiple documents and tokens can attend across document boundaries.
Noticeably, the packing process implies that the second half-epoch of deduplicated data contains the same documents but not necessarily the same sequences.
There does not exist an official validation set for Pythia models. However, we confirmed with the authors that the original Pile validation set has not been used for training.

\paragraph{Models.}
The Pythia model suite is composed of 16 models: transformers of \integer{8} different sizes trained on the Pile as-is or deduplicated.
All model sizes were trained using a cosine learning rate schedule with warm-up, the same data order, and a batch size of \integer{1024} sequences, resulting in exactly \q{143}{\thousand} optimisation steps.
The final \q{48}{\thousand} optimisation steps correspond to the second half-epoch.
Thus, we focus on model checkpoints at initialisation (step \integer{0}), and after every \q{1}{\thousand} iterations (steps \q{1}{\thousand}-\q{95}{\thousand}) resulting in \integer{96} checkpoints evenly spaced throughout training. For completeness, we report the second half-epoch (steps \q{96}{\thousand}-\q{143}{\thousand}) analysis in \cref{app:additional_plots}. Additionally, log-spaced checkpoints are available for timesteps early in training (timesteps $\checkpointtime \in \{2^i\}_{i=0}^{9}$).
We do not consider them to obtain evaluations at evenly spaced intervals.
We use all available model sizes, that is, \q{70}{\million}, \q{160}{\million}, \q{410}{\million}, \q{1.4}{\billion}, \q{6.9}{\billion}, and \q{12}{\billion}, except \q{2.8}{\billion}. We exclude \q{2.8}{\billion} from the results since we found a potential mismatch between the available checkpoints and the data order used during training.

\subsection{Hardware Details}

We use a server with one \makesf{NVIDIA A100 80GB PCIe}, \integer{32} CPUs, and \integer{32} GB of RAM for all experiments. Below, we report a subset of the output of the \makesf{lscpu} command:

\begin{tcolorbox}[left=5pt,right=5pt,top=5pt,bottom=5pt]
    \small
    \begin{verbatim}
Architecture:        x86_64
CPU op-mode(s):      32-bit, 64-bit
Address sizes:       46 bits physical, 
                     48 bits virtual
Byte Order:          Little Endian
CPU(s):              32
On-line CPU(s) list: 0-31
Vendor ID:           GenuineIntel
Model name:          Intel(R) Xeon(R)
                     Silver 4210R CPU
                     @ 2.40GHz
CPU family:          6
Model:               85
Thread(s) per core:  1
Core(s) per socket:  1
Socket(s):           8
Stepping:            7
BogoMIPS:            4800.11
\end{verbatim}
\end{tcolorbox}

\section{Additional Results}\label{app:additional_plots}

On the next page in \cref{fig:other_metrics_full}, we report the memorisation profiles obtained using other metrics besides sequence-level log-likelihood. Specifically, the average token-level accuracy given the true context and the average rank assigned by the model to the correct next token given the true context.
We report the results for the entire training process---i.e., using all the available checkpoints: $\checkpointtime\in\{\integer{0}, \q{1}{\thousand}, \smalldots, \q{143}{\thousand}\}$ and $\cohorttime\in\{\q{1}{\thousand}, \smalldots, \q{143}{\thousand}\}$.
We present the metrics from most coarse---i.e., average token accuracy (\cref{subfig:average_token_accuracy})---to most fine-grained---i.e., sequence log-likelihood (\cref{subfig:sequence_loglikelihood}).

As shown in \cref{fig:other_metrics_full}, different performance metrics result in distinct memorisation estimates. Specifically, finer-grained metrics---like sequence log-likelihood---allow us to detect smaller memorisation effects, and vice versa.
For example, average token accuracy, which is the most coarse-grained metric, mostly does not capture instantaneous memorisation for Pythia \q{410}{\million}. Instead, a finer-grained metric---like average token rank or sequence log-likelihood---detects additional effects.
Depending on the use-case different metrics might be appropriate.
For example, analogously to extractable memorisation \citep{carlini-etal-2021-extracting}, average token accuracy could be used to measure memorisation as it matches the specific use-case: detecting whether a model would generate a specific sequence when prompted with some of its tokens.
We chose sequence log-likelihood because it allows us to capture more fine-grained memorisation effects beyond the capability of the model to generate a specific sequence.
In particular, accuracy captures \enquote{hard} transitions in the model's output by determining whether a token is the most likely in the model's output distribution.
Log-likelihood, on the other hand, captures more nuanced impacts of training on an instance by assessing whether a token is more likely to be generated than it would be otherwise.

\begin{figure}[htbp]
    \centering
    \subfloat[\textbf{Average Token Accuracy:},
    $\perffn(\vtheta_{\checkpointtime}, \xdatainstance) = \frac{1}{\size{\xdatainstance}} \sum_{i=1}^{\size{\xdatainstance}}\mathbbm{1}(\hat{x}_i = {\color{\instancecolor} x}_i)$, where $\hat{x}_i = \argmax_{x\in\vocab} \pthetacheckpoint(x\mid\xdatainstance_{<i})$ is the predicted token at position $i$ computed using the correct previous tokens as context and $\size{\xdatainstance}$ is the number of tokens in the sequence. \label{subfig:average_token_accuracy}]{\includegraphics[width=\columnwidth]{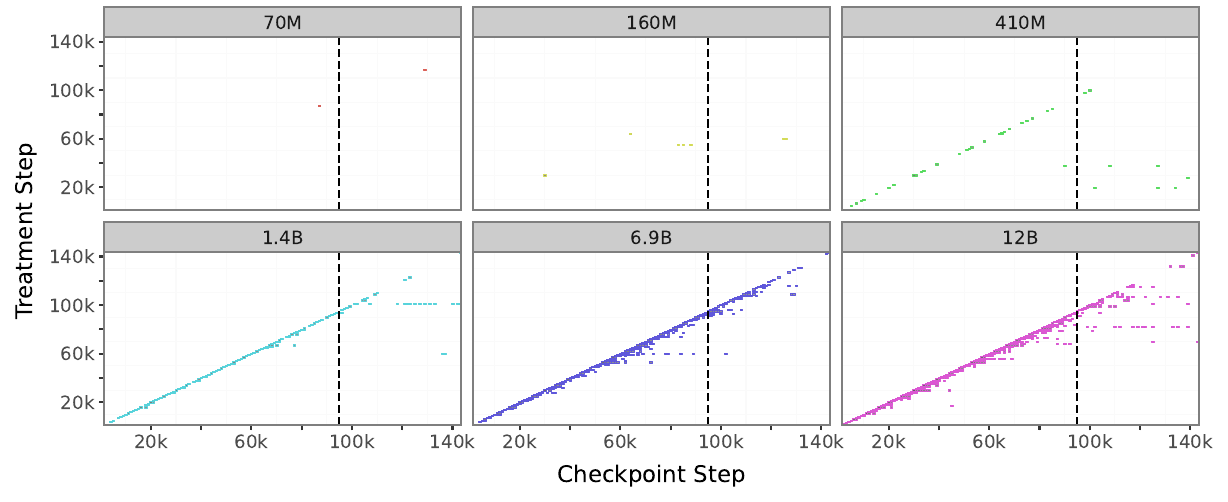}} \\
    \subfloat[\textbf{Average Rank of the True Token:}  $\perffn(\vtheta_{\checkpointtime}, \xdatainstance) = \frac{1}{\size{\xdatainstance}} \sum_{i=1}^{\size{\xdatainstance}} \mathrm{rank}({\color{\instancecolor} x}_i)$, where the function $\mathrm{rank}(\cdot)$ returns the rank of the true token at position $i$ computed from the probabilities assigned by the model using the correct previous tokens as context, i.e. $\pthetacheckpoint(x\mid\xdatainstance_{<i})$, and $\size{\xdatainstance}$ is the number of tokens in the sequence.]{\includegraphics[width=\columnwidth]{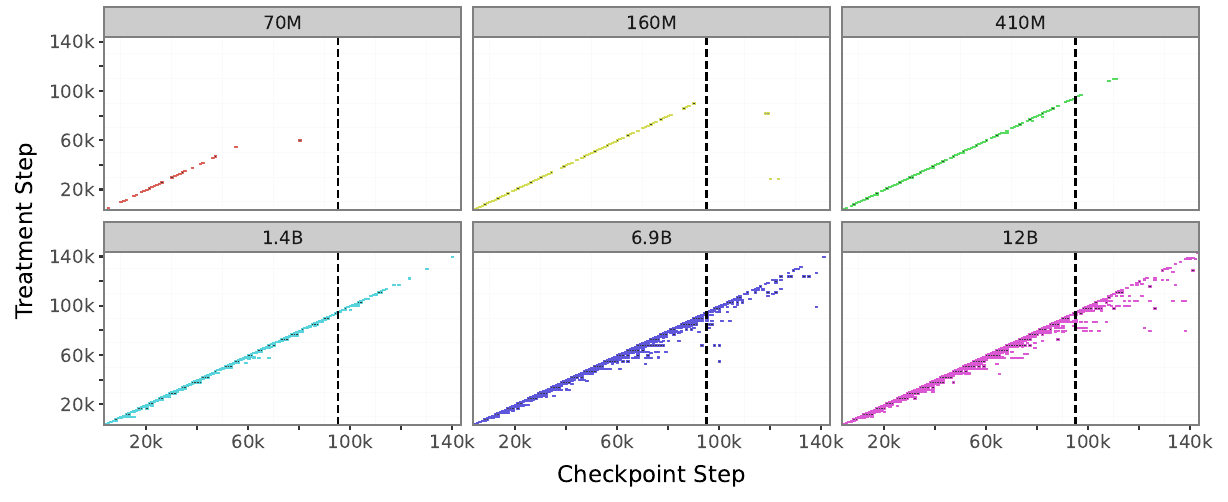}} \\
    \subfloat[\textbf{Sequence Log-Likelihood:}  $\perffn(\vtheta_{\checkpointtime}, \xdatainstance) = \log \pthetacheckpoint(\xdatainstance)$, where $\log \pthetacheckpoint(\xdatainstance) = \sum_{i=1}^{\size{\xdatainstance}} \log \pthetacheckpoint({\color{\instancecolor} x}_i\mid\xdatainstance_{<i})$ and $\size{\xdatainstance}$ is the number of tokens in the sequence. \label{subfig:sequence_loglikelihood}]{\includegraphics[width=\columnwidth]{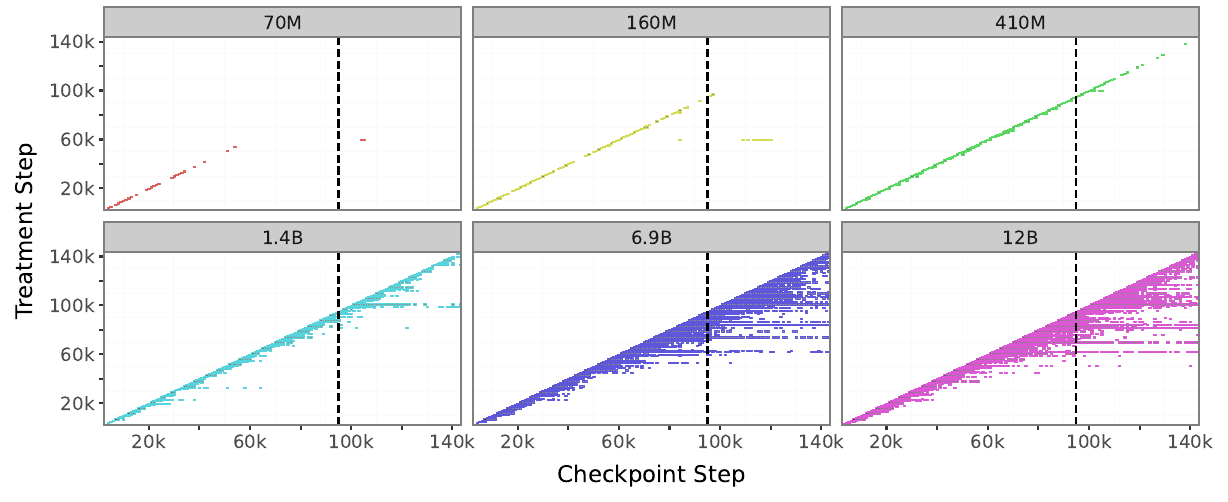}} \\
    \caption{Memorisation profiles ($\memorisationcohort$) computed using different performance metrics $\perffn$ using all the available checkpoints---i.e., $\checkpointtime\in\{\integer{0}, \q{1}{\thousand}, \smalldots, \q{143}{\thousand}\}$ and $\cohorttime\in\{\q{1}{\thousand}, \smalldots, \q{143}{\thousand}\}$. The dashed vertical line indicates the end of the first epoch ($\checkpointtime\mathop{=}\q{95}{\thousand}$). We only show statistically significant entries. \label{fig:other_metrics_full}}
\end{figure}

\end{document}